%% file: arxiv_main.tex
\pdfoutput=1
\PassOptionsToPackage{dvipsnames}{xcolor}
\documentclass{article} %
\usepackage{iclr2027_conference,times}
\usepackage{geometry}
\input{math_commands.tex}

\usepackage[utf8]{inputenc}
\usepackage[T1]{fontenc}
\usepackage{hyperref}
\usepackage{url}
\usepackage{amssymb}
\usepackage{graphicx}
\usepackage[dvipsnames]{xcolor}
\usepackage[most]{tcolorbox}
\usepackage{booktabs}
\usepackage{multirow}
\usepackage{subcaption}
\usepackage{amsthm}
\usepackage{arydshln}
\newtheorem{theorem}{Theorem}

\definecolor{PromptGreen}{HTML}{1A531A}
\definecolor{PromptGreenBg}{HTML}{F0F9F0}

\title{Self-Evolving Agents via Likelihood-Guided Tool-Space Optimization}

\author{Xuanqi Zhang$^{1,2}$, Ruinan Jin$^{1,2}$, Running Yang$^{1}$,  Yuxuan Zhang$^{1}$, Minghui Chen$^{3}$, \\ \textbf{Wenlong Deng}$^{1,\star,\dag}$, \textbf{Xiaoxiao Li}$^{1,2,\dag}$
 \vspace{8pt}
 \\
 $^1$University of British Columbia,
 $^2$Vector Institute,
 $^3$Columbia University
\vspace{3pt} \\
$^\dag$Corresponding author, $^\star$ Project Leader
}

\iclrfinalcopy

\usepackage{caption}
\usepackage[shortlabels]{enumitem}
\usepackage{wrapfig}
\newsavebox{\fitbox}
\newcommand{\fitwidth}[1]{\sbox{\fitbox}{#1}\ifdim\wd\fitbox>\linewidth\resizebox{\linewidth}{!}{\usebox{\fitbox}}\else\usebox{\fitbox}\fi}
\usepackage{placeins}
\usepackage{float}

\newcommand{\ours}{$\texttt{LOTS}$}
\begin{document}
\addtocontents{toc}{\protect\setcounter{tocdepth}{-1}}

\maketitle

\begin{abstract}
Self-evolving agents can continually improve their behavior, while tools define the executable action space through which they interact with the environment. However, exposing the full tool library to a model introduces substantial irrelevant context and can impair tool-use decisions.
We study tool-space self-evolution, where each recurring task type maintains a persistent tool space which is constructed from accumulated output experience. We identify three limitations of existing methods:
(1) \emph{output-unaware selection}: they rely primarily on tool descriptions or model priors rather than observed tool outputs;
(2) \emph{statelessness across requests}: they select tools independently for each request without consolidating prior output experience into persistent task-specific state;
(3) \emph{inference cost}: they repeatedly search, rank, or reason over candidate tools for subsequent requests of the same task.
We address these limitations through output-aware tool scoring, persistent task-specific tool spaces, amortized tool selection, and reusable configurations across models.
We introduce \ours{} (\emph{Likelihood-Only Tool Scoring}), which evolves an agent's tool space from accumulated output experience while keeping model parameters fixed.
After each request, \ours{} holds the model's generated answer and estimates each tool's contribution by measuring how much the answer likelihood changes when its observed output is removed.
These contributions are aggregated within each recurring task to rank tools and update its persistent space.
Across three benchmarks, \ours{} improves task performance while substantially reducing tool context.
More importantly, sequential experiments demonstrate that task-specific spaces persist and continue to improve over time, while cross-model experiments show that learned configurations transfer across different models.
\end{abstract}

\section{Introduction}
\vspace{-1mm}
\label{sec:intro}
Recent advances in large language models (LLMs) have enabled increasingly
capable autonomous agents that can reason~\citep{hurst2024gpt}, invoke external
tools~\citep{deng2025group}, and operate in complex environments
~\citep{claude_toolsearch}. As these systems become more autonomous, an important capability is to improve from accumulated interaction experience
rather than treating every request independently~\citep{selfevolving2025}.
Such \emph{self-evolution} involves persistent updates to model parameters~\citep{zhai2025agentevolver,toolr0,searl,evotool},
contextual state such as prompts and memories~\citep{ace},
tools and reusable skills~\citep{voyager,skillopt},
or control architectures~\citep{selfevolving2025,selfimprove2026}.

Among these components, the \emph{tool space} is particularly important yet under-explored. It determines the agent's executable action space, while tool descriptions and schemas also constitute part of the context processed by the model at inference time. Unlike experience that can be stored and selectively retrieved when relevant, tool specifications are typically authored in advance and exposed to the model at request time~\citep{toolformer,gorilla,qin2024toolllm}. Consequently, as the global tool library expands, exposing it in full imposes a recurring cost: irrelevant tools occupy context even when they are not needed and enlarge the set of alternatives among which the model must choose. Prior work has shown that irrelevant or increasingly long context can impair language models' ability to identify and use relevant information reliably~\citep{lostinmiddle}. This creates a structural mismatch: the global tool library may continually accumulate capabilities, whereas each recurring task typically requires only a small, task-specific subset. As shown in Figure~\ref{fig:teaser}, exposing an expanding global tool space in full increases context overhead and reduces accuracy, even when task requirements remain unchanged. In contrast, maintaining a task-specific tool space substantially reduces context overhead and preserves higher accuracy as the global tool space grows. Additional results in Appendix~\ref{app:registry} (Figure~\ref{fig:tgb-scale}) further support this finding.  

Existing work on tool-oriented self-evolution has largely focused on
creating new tools~\citep{evosop,smith} or adapting their documentation,
while comparatively little attention has been paid to evolving
task-specific tool spaces from execution
experience~\citep{selfevolving2025}. 
More related works are in Appendix~\ref{sec:related}.
A natural starting point for constructing such a space is tool selection:
determining which tools to retain from a candidate set.
Existing approaches broadly follow two strategies, \emph{Search-based} and
\emph{LLM-based methods} ~\citep{gorilla,qin2024toolllm,colt,hyset,toolrerank,anytool,tool2vec,mcpzero}.
\emph{Search-based methods} retrieve tools by matching requests to tool
descriptions~\citep{qin2024toolllm}, but semantic relevance alone does not
capture how effectively the serving model can use each tool.
\emph{LLM-based methods} use a language model to reason over candidate
tools~\citep{kachuee2025querygen}, or train it to select and reconfigure
tools during execution~\citep{autotool, zhou2026toolself}.
These methods draw on the selector's understanding of tool functionality or model prior.
Our focus is complementary: how can execution experience refine a persistent, task-specific tool space in an output-aware manner? To our knowledge, TTO~\citep{octotools} is the only prior method using output accuracy for tool selection, but it requires ground-truth labels and costly greedy search over tool compositions (as shown in Fig.~\ref{fig:cost}).
In summary, prior methods suffer from one or more of three key limitations:
(1) \emph{output-unaware}:
Selection is primarily based on tool descriptions or model priors before the
tool has run, and therefore does not directly measure whether the tool's
observed output actually supported the resulting answer.
(2) \emph{statelessness across requests}:
A tool-selection decision is typically made independently for each request.
Output experience from previous requests of the same recurring task is not
consolidated into persistent state that can shape how later requests are
served.
(3) \emph{costly at inference time}:
As the global registry (tool space) expands, tools must be repeatedly searched, ranked, or
reasoned over for each new request, even when the deployed agent has already seen
many requests from the same task.
This motivates our central question:
\begin{center}
\begin{tcolorbox}[
reset,
colback=gray!5,
colframe=black!60,
boxrule=0.5pt,
toprule=0.5pt,
bottomrule=0.5pt,
leftrule=0.5pt,
rightrule=0.5pt,
sharp corners,
left=8pt,
right=8pt,
top=4pt,
bottom=4pt,
width=0.85\textwidth
]
\centering
\emph{Can an agent evolve a persistent tool space from its own output
experience, so that later requests are served more accurately and efficiently?}
\end{tcolorbox}
\end{center}
\begin{figure}[t]
\centering
\includegraphics[width=\linewidth]{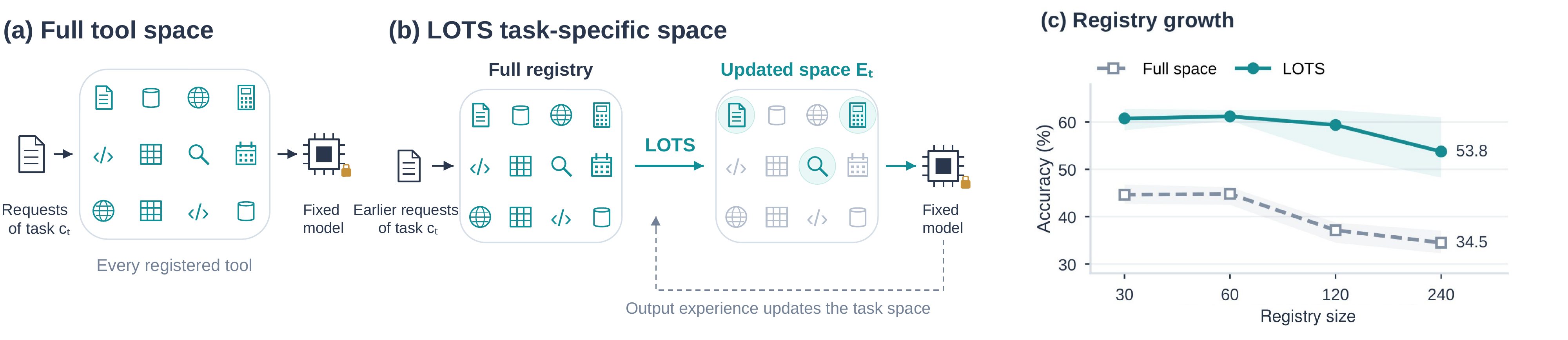}
\vspace{-4mm}
\caption{\textbf{A task-specific tool space that evolves from the agent's own outputs.}
\textbf{(a)} A fixed model is served full registry on every request of a task.
\textbf{(b)} \textit{Likelihood-Only Tool Scoring} (\ours{}) turns the output experience of a task's earlier requests into a compact space $E_t$. Later requests of the task reuse it, and their outputs update it.
\textbf{(c)} Accuracy as the registry grows from 30 to 240 tools (TGB, Qwen3.5-9B): serving every tool falls to 34.5\%, while the \ours{} space stays at 53.8\% or above (Appendix~\ref{app:registry}).}
\label{fig:teaser}
\vspace{-4mm}
\end{figure}

To answer this question, we introduce \emph{Likelihood-Only Tool Scoring} (\ours{}), which evolves persistent task-specific tool spaces from output experience. \ours{} first makes tool optimization \emph{output-aware}: it freezes the model's realized output and measures its likelihood change under a benchmark-specific intervention
that removes evidence associated with each tool. It then aggregates these request-level contributions across recurring requests of the same task, ranks the tools by their accumulated evidence, and converts the ranking into a persistent task-specific space through tool pruning. This persistent space is reused to serve subsequent requests, whose new output traces can in turn revise the configuration. As a result, output evidence is retained across interactions and missing tools can be captured in the successive turns. We evaluate this design on TGB, BFCL, and GTA-Atomic: on TGB, a controlled compositional benchmark, \ours{} keeps 2.8 to 6.8 of 15 tools on average and improves all eight evaluated models over serving every tool, by 3.8 to 25.0 accuracy points; on BFCL, \ours{} improves all five models; and on GTA-Atomic, tool space fitted on earlier requests improves six of seven models when later requests are executed by a ReAct agent with executable tools. The learned spaces also remain revisable as new traces arrive and can be reused across models.
Our contributions are:
\vspace{1mm}
\\
\noindent$\bullet$ \textbf{Output-aware tool optimization.}
\ours{} values tools from their observed outputs, using the likelihood
change of the model's realized answer under leave-one-out intervention.
This moves beyond description- or prior-based selection and yields tool spaces
that generally improve over the full registry across the evaluated benchmarks and models.
\vspace{1mm}
\\
\noindent$\bullet$ \textbf{Persistent task-specific evolution.}
Rather than selecting tools independently for each request, \ours{} aggregates
output evidence across recurring requests into a persistent task-specific
space and is reused by later requests and revised as new traces arrive.
\vspace{1mm}
\\
\noindent$\bullet$ \textbf{Reduced serving context.}
\ours{} reuses compact task-specific spaces instead of searching or reasoning over the full registry for every request, reducing recurring tool context while maintaining or improving performance; Appendix~\ref{app:efficiency} reports the one-time construction cost and the number of requests needed to recover it.
\vspace{1mm}
\\
\noindent$\bullet$ \textbf{Tool-space reuse across models.}
As an additional benefit, the optimized tool spaces learned from one model can often
be reused by models from other scales and families.
\vspace{-1mm}
\section{Problem Formulation}
\vspace{-1mm}
\label{sec:problem}
\textbf{Tool-use interaction.}
Let $\mathcal{S}$ denote the global tool registry, and
$E \subseteq \mathcal{S}$ is the tool set available to the agent.
Given a request $x$, the initial context $h_0$ contains the request and the specifications of tools in $E$.
At each interaction step \(\ell\), the model generates an action
$a_\ell \sim p_\theta(\cdot \mid h_{\ell-1}),$
where \(a_\ell\) is either a terminal response or a tool call \((s_\ell,z_\ell)\) specifying a tool \(s_\ell\in E\) and its arguments \(z_\ell\).
A tool call produces an observation \(o_\ell=s_\ell(z_\ell)\), which is appended to the context: $
h_\ell=h_{\ell-1}\oplus(a_\ell,o_\ell).$ 
The agent continues this action--observation process until it produces the final answer \(y_\theta(x;E)\). The tool space therefore determines both the agent's executable action space and the tool information available during inference. 
\\
\textbf{Continual tool-space evolution.}
We consider a fixed language-model agent \(p_\theta\) deployed with the global tool registry \(\mathcal{S}\). The agent continually self-evolves as a sequence of tasks \(c_1,\ldots,c_T\) arrives: for each incoming task \(c_t\), it optimizes a task-specific tool space \(E_t\subseteq\mathcal{S}\), whose optimum is
\[
E_t^\star=\arg\max_{E\subseteq\mathcal{S}}\ \mathcal{U}(c_t;p_\theta,E),
\]
where \(\mathcal{U}\) measures task performance. As new tasks arrive, the agent continually expands its collection of task-specific spaces \(\{E_1,\ldots,E_T\}\), and a task's space can be revised as more of its requests are served, improving performance through tool-space optimization while keeping \(p_\theta\) fixed.

\vspace{-1mm}
\section{Method}
\label{sec:method}
\vspace{-1mm}
\begin{figure}[t]
\centering
\includegraphics[width=0.95\linewidth]{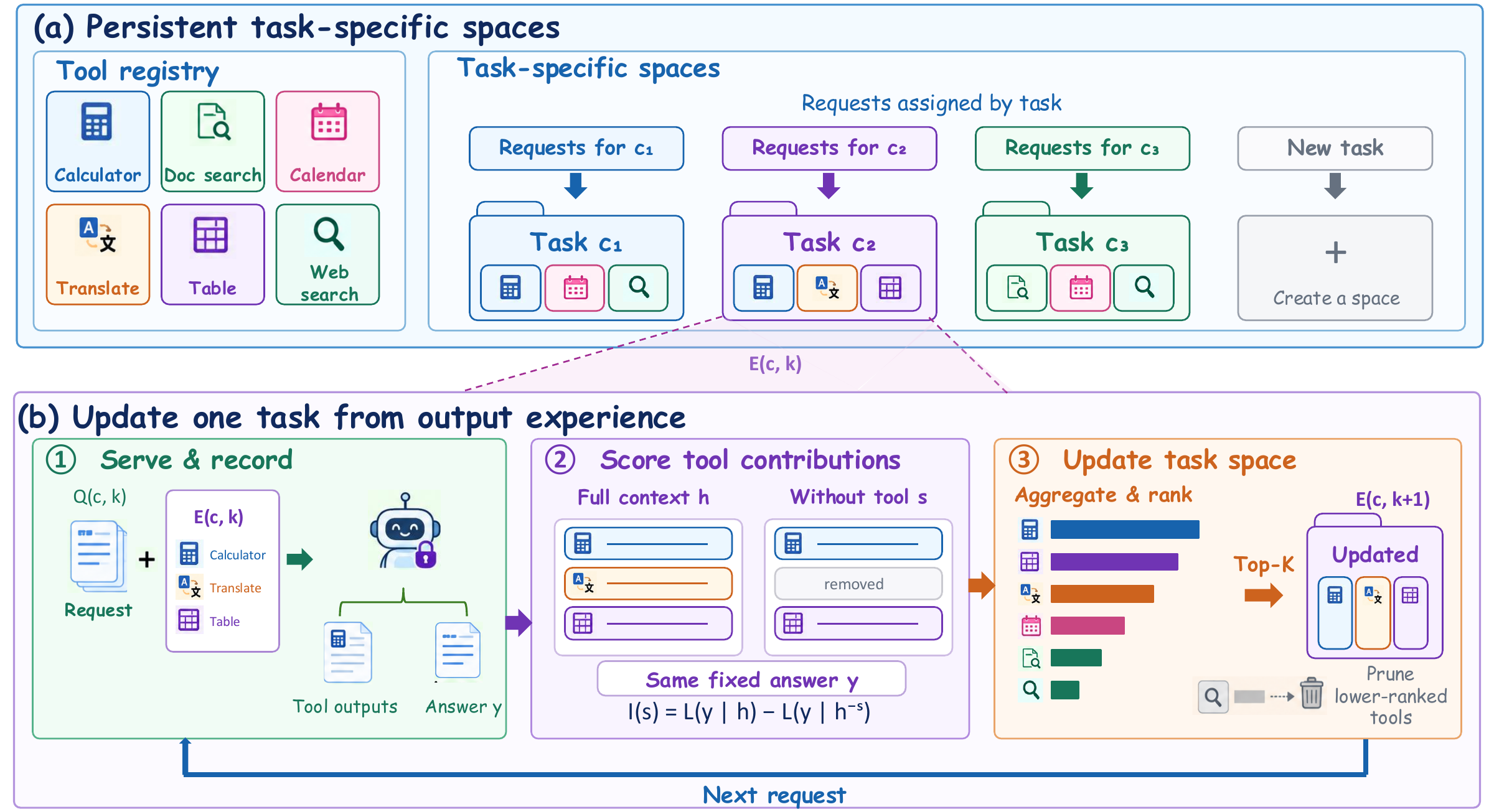}
\caption{
\textbf{\ours{} evolves persistent task-specific tool spaces from output
traces.}
\textbf{(a)} Requests are assigned to recurring task types, each of which
maintains its own space $E_{c,k}$. Requests of task $c$ collected for an
update are served under the current space.
\textbf{(b)} One \ours{} update freezes each generated output, scores tools
by leave-one-out likelihood, aggregates the scores within the task, and
uses the ranking to prune the task's tool space.
}
\label{fig:overview}
\vspace{-4mm}
\end{figure}
We introduce \ours{} for
continually constructing task-specific tool spaces from an agent's own output
experience while keeping the model parameters fixed. As tasks
$c_1,\ldots,c_T$ arrive sequentially, each incoming task $c_t$ provides a set
of requests $Q_t$. As shown in Fig.~\ref{fig:overview}, the agent first serves these requests under a candidate tool
space and records the resulting tool-use traces. \ours{} then
(1) measures how much each observed tool output supports the model's realized
answer, (2) aggregates this evidence across requests of the same task to obtain
a task-level tool ranking, and (3) converts the ranking into a persistent
task-specific tool space $E_t$. The resulting space is stored and reused when
serving subsequent requests of $c_t$. As new tasks arrive, the deployed agent
therefore continually expands its collection of task-specific spaces
$\{E_1,\ldots,E_t\}$ without modifying the underlying model $p_\theta$. 
Figure~\ref{fig:overview} summarizes the procedure; implementation details are given in Appendix~\ref{app:method-details}.
\ours{} can also be applied to a mix of tasks in a batch, with details in Appendix~\ref{app:routing}.

\vspace{-1mm}
\subsection{LOTS: Output-Aware Tool Scoring Via Likelihood Guidance}\label{sec:lots-a} %
\vspace{-1mm}
For an incoming task $c_t$, let $ Q_t=\{x_1,\ldots,x_{n_t}\}$
denote the requests used to construct its tool space. Consider a request
$x\in Q_t$ served by the agent, and let $y_x$ denote the model's realized
output. We write $h_x$ for the recorded context from which $y_x$ is produced,
including the request, available tool documentation, tool calls, and observed
tool outputs. Our goal is to measure the contribution of each tool from what it
\emph{actually produced during execution}. For a tool $s$, we construct
an intervened context $h_x^{-s}$ in which the output evidence attributable to
$s$ is removed from the recorded execution. We then keep the realized output
$y_x$ fixed and define the request-level contribution of $s$ as
\begin{equation}
I_x(s)
=
\mathcal{L}_{\theta}(y_x \mid h_x)
-
\mathcal{L}_{\theta}(y_x \mid h_x^{-s}),
\label{eq:tool_score}
\end{equation}
where
\begin{equation}
\mathcal{L}_{\theta}(y\mid h)
=
\frac{1}{|y|}
\sum_{j=1}^{|y|}
\log p_\theta(y_j\mid h,y_{<j})
\label{eq:output_likelihood}
\end{equation}
is the mean log-likelihood over the scored output tokens. A positive $I_x(s)$ means that removing the observed evidence from tool $s$
reduces the model's support for the same realized output, indicating that the tool contributed useful evidence to that execution.
\\
Two design choices are important. First, \ours{} evaluates the \emph{same frozen output} under both contexts, measuring how tool evidence changes model support without additional rollouts. Second, each tool is evaluated alongside the other tools used in the execution, preserving dependencies in compositional tool use. \ours{} does not observe correctness, but measures which tool evidence the model relies on. Empirically, correct and incorrect outputs often yield similar task-level rankings. Appendix~\ref{app:diagnostics} examines when this holds or breaks down. Benchmark-specific interventions and scored tokens are also detailed in Appendix~\ref{app:implementation}. 
\vspace{-1mm}
\subsection{Task-Level Tool-Space Construction}
\label{sec:lots-update}
\vspace{-1mm}
The score $I_x(s)$ measures the contribution of tool $s$ to a single request.
To construct a tool space for the recurring task $c_t$, \ours{} aggregates
these request-level contributions across the task's output experience. Let
$
Q_t(s)
=
\{x\in Q_t : s \text{ is exposed and scoreable for } x\}
$
denote the requests that provide evidence for tool $s$. Its
task-level contribution
\begin{equation}
\bar I_t(s)
=
\frac{1}{|Q_t(s)|}
\sum_{x\in Q_t(s)} I_x(s).
\label{eq:task_score}
\end{equation}
Thus, evidence collected from otherwise independent requests is consolidated
into a shared task-level estimate rather than discarded after each
interaction. A tool that is not exposed on a request contributes no evidence
for that request. Let $A_t\subseteq S$ denote the tools eligible for optimization. \ours{} ranks
them according to $\bar I_t(s)$ and defines the high-value set
\begin{equation}
E_t
=
\operatorname{TopK}_{s\in A_t} \bar I_t(s),
\label{eq:topk}
\end{equation}

where $K$ specifies the tool-space budget, with its protocol
described in Appendix~\ref{app:budget}.
Fix a common menu $S$ with $|S|=m$ and the recorded outputs.
Let $F(A)$ denote the mean fixed-output likelihood when only
the evidence of $A\subseteq S$ is retained, so that
$\bar I_t(s)=F(S)-F(S\setminus\{s\})$, and define
\[
\widehat F(A)
=
F(S)-\sum_{s\notin A}\bar I_t(s).
\]
The following result relates top-$K$ selection to maximizing $F$
under bounded tool interactions, a detailed proof is provided in Appendix~\ref{app:theory}.
\begin{theorem}[Selection under bounded interactions]
\label{thm:selection}
Let $r=m-K$. Suppose
$
|\Delta_{s,u}F(B)|\le\beta,
$
for every $B\subseteq S$ with $K\le|B|\le m-2$
and distinct $s,u\notin B$, where
$
\Delta_{s,u}F(B)
=
F(B\cup\{s,u\})-F(B\cup\{s\})
-F(B\cup\{u\})+F(B)
$
is the second-order difference
\citep{sundararajan2020shapley}.
Then $E_t$ in \eqref{eq:topk} maximizes $\widehat F(A)$
over $|A|=K$, and, for every such $A$,
$
|F(A)-\widehat F(A)|\le\binom{r}{2}\beta.
$
Consequently, for
$A^\star\in\arg\max_{|A|=K}F(A)$,
$
0\le F(A^\star)-F(E_t)\le r(r-1)\beta.
$
Both bounds are zero for $r\le 1$. Proofs are in Appendix~\ref{app:theory}.
\vspace{-1mm}
\end{theorem}

\vspace{-1mm}
\subsection{Continual Task-Specific Tool-Space Evolution}
\vspace{-1mm}
The constructed space $E_t$ is not a request-local selection. It is stored as
persistent deployment state associated with task $c_t$ and reused for all
requests of the same task. When the next task arrives, \ours{} independently
constructs its corresponding space from that task's own output experience.
After observing tasks $c_1,\ldots,c_t$, the deployed agent therefore maintains
$\mathcal{E}_t = \{E_1,\ldots,E_t\},$
and the continual update is
\begin{equation}
\mathcal{E}_{t-1}
\;\longrightarrow\;
\mathcal{E}_{t}
=
\mathcal{E}_{t-1}\cup\{E_t\}.
\label{eq:continual_space}
\end{equation}
Earlier task spaces remain unchanged when a new task arrives. Consequently,
new task knowledge is incorporated by expanding the collection of
task-specific tool spaces rather than by modifying $p_\theta$ or repeatedly
reconstructing a tool configuration for every request. In this sense, the
deployed agent self-evolves through its tool-space state:
\[
c_t
\;\longrightarrow\;
Q_t
\;\longrightarrow\;
\{I_x(s)\}
\;\longrightarrow\;
\bar I_t(s)
\;\longrightarrow\;
E_t.
\]
\begin{wraptable}{r}{0.49\textwidth}
\vspace{-4mm}
\centering
\footnotesize
\renewcommand{\arraystretch}{1.08}
\begin{tabular}{@{}p{\linewidth}@{}}
\toprule
\textbf{Algorithm 1: Task-specific tool-space construction} \\
\midrule

\textbf{Input:}
Task $c_t$, request batch $Q_t$, candidate space
$E_t^{\mathrm{c}}\subseteq\mathcal{S}$,
budget $K$, fixed model $p_\theta$. \\[3pt]

\textbf{1. Serve.}
Execute $Q_t$ under the fixed space $E_t^{\mathrm{c}}$
and record traces $(h_x,y_x)$. \\[3pt]

\textbf{2. Score.}
For each request $x$ and exposed, scoreable tool $s$,
keep $y_x$ fixed and compute
$I_x(s)=L_\theta(y_x\mid h_x)
-L_\theta(y_x\mid h_x^{-s})$.
Here $h_x^{-s}$ removes $s$'s calls and observations; where a later
tool consumes $s$'s output, the chain is re-executed without $s$
(Appendix~\ref{app:scoring-impl}). \\[3pt]

\textbf{3. Aggregate.}
Let $Q_t(s)$ contain the requests providing scores for $s$,
and $A_t=\{s\in E_t^{\mathrm{c}}:|Q_t(s)|>0\}$.
Compute
$\bar I_t(s)=
\frac{1}{|Q_t(s)|}\sum_{x\in Q_t(s)}I_x(s)$. \\[3pt]

\textbf{4. Prune.}
Set $E_t=\operatorname{TopK}_{s\in A_t}\bar I_t(s)$,
retaining all eligible tools if $|A_t|<K$. \\[3pt]

\textbf{5. Deploy.}
Store $E_t$ for subsequent task-$c_t$ requests;
repeat with new evidence when updating. \\
\bottomrule
\end{tabular}
\vspace{-4mm}
\label{alg:lots}
\end{wraptable}
If additional experience from an existing task becomes available, its stored
space can be revised using the same scoring and aggregation procedure. Denoting
successive versions by $E_t^{(0)},E_t^{(1)},\ldots$, a new batch of requests
from $c_t$ produces
\begin{equation}
E_t^{(k)}
\;\xrightarrow{\;\text{new output experience}\;}
E_t^{(k+1)}.
\label{eq:within_task}
\end{equation}
This within-task refinement is optional to the continual task-arrival setting:
the core mechanism only requires that each incoming task constructs and
retains its own persistent space. Later batches of a task are served under
its current space, their traces update the running scores, and the space is
recomputed from them. Details of the running-score updates are given in
Appendix~\ref{app:aggregation}.

\vspace{-1mm}
\section{Experiments}
\label{sec:experiments}
\vspace{-1mm}
We evaluate task performance (\S\ref{sec:exp-shared}),
evolution across and within tasks (\S\ref{sec:exp-loop}),
and serving efficiency as well as cross-model reuse
(\S\ref{sec:exp-eff}).
Scoring analyses and ablations are presented in
Appendix~\ref{app:diagnostics}.
\vspace{-1mm}

\subsection{Experimental Setup}
\label{sec:experimental-setup}
\label{sec:exp-setup}
\vspace{-1mm}
Following \S\ref{sec:problem}, each task is a TGB family, a BFCL API class, or a GTA-Atomic task type.
 We evaluate the top-$K$ spaces of the \ours{} ranking over a predefined grid and keep the
smallest $K$ attaining the highest fitting accuracy (Appendix~\ref{app:budget}).
Likelihood scoring does not use correctness labels,
whereas only budget selection in these experiments uses fitting labels.
\begin{table}[!tb]
\centering
\caption{
\textbf{TGB dataset accuracy (\%).}
One tool space fitted per task family.
Search-based baselines include BM25, Dense, and Tool2Vec;
the LLM-based router selects tools per request.
$^\dagger$ marks methods that read correctness labels on the fitting requests.
Bold marks the best result in each column.
}
\label{tab:tgb-family}
\scriptsize
\setlength{\tabcolsep}{2.6pt}
\begin{tabular}{@{}lcccccccc@{}}
\toprule
& Qwen2.5-7B & Qwen2.5-14B & Llama3.1-8B & Mistral-7B
& Qwen3.5-9B & Qwen3-8B & Phi4-14B & Gemma4-12B \\
\midrule
No tools
& 10.8 & 11.4 & 12.9 & 10.0 & 10.5 & 10.4 & 11.0 & 12.8 \\
Keep all
& 69.6 & 88.0 & 77.7 & 46.4 & 50.9 & 56.5 & 79.7 & 86.7 \\
Random
& 23.4 & 19.3 & 41.3 & 11.5 & 27.5 & 10.1 & 24.3 & 20.0 \\
\midrule
\multicolumn{2}{@{}l}{\textit{Search-Based}}  &  &  &  & \\
BM25
& 19.3 & 36.3 & 40.0 & 13.4 & 24.9 & 10.0 & 39.3 & 37.5 \\
Dense
& 20.5 & 36.3 & 37.6 & 17.2 & 18.5 & 18.5 & 11.1 & 38.8 \\
Tool2Vec
& 38.1 & 45.3 & 57.4 & 24.2 & 23.9 & 37.9 & 55.6 & 47.5 \\
\midrule
\multicolumn{2}{@{}l}{\textit{LLM-Based}}  &  &  &  & \\
LLM as router
& 31.7 & 62.5 & 39.4 & 29.3 & 36.1 & 44.7 & 61.7 & 66.9 \\
Jev
& 60.3 & 70.8 & 68.2 & 51.5 & 56.6 & 49.5 & 81.7 & 71.3 \\
\midrule
Beam Search$^\dagger$ & 26.6 & 38.1 & 19.6 & 21.2 & 23.8 & 36.8 & 64.1 & 30.0 \\
TTO$^\dagger$
& 48.6 & 88.0 & 65.4 & 63.1 & 63.7 & 63.0 & 72.1 & 86.8 \\
\hdashline
\ours{}
& \textbf{82.9} & \textbf{91.8} & \textbf{84.8} & \textbf{65.8}
& \textbf{65.2} & \textbf{81.5} & \textbf{92.1} & \textbf{91.4} \\
\bottomrule
\end{tabular}
\vspace{-3mm}
\end{table}
\\
\textbf{Benchmarks.} We evaluate \ours{} on \emph{three} benchmarks.
\textbf{TGB (controlled composition)}, provides a controlled compositional setting for studying
tool dependencies and distractors, and contains
4{,}000 requests across ten families and fifteen tools
(seven task tools and eight distractors); Appendix~\ref{app:datasets} gives its construction and the family chains (Table~\ref{tab:tgb-families}), and Appendix~\ref{app:worked} follows one TGB task through the whole procedure.
Each task uses 80 requests for fitting and 320 for evaluation,
enabling controlled studies of tool dependencies, distractors,
and space updates. 
\textbf{BFCL-v4 multi-turn (multi-turn calling})~\citep{bfcl} evaluates multi-turn function calling
with task-specific documentation allocation.
Success requires completing the required call sequence correctly;
each API class uses 10 records for fitting and 40 for evaluation.
\textbf{GTA-Atomic (executable ReAct agent)}, the atomic-task split of GTA~\citep{gta}, evaluates ReAct task execution through the official harness on three annotated task types. 
Table~\ref{tab:gta-combined} reports accuracy over the 112 evaluation requests of the three types pooled, from one serving session per model.
Appendices~\ref{app:datasets} and~\ref{app:scoring-impl}
detail the datasets, tool libraries, and scoring interventions.
\\
\noindent\textbf{Models.}
We evaluate Qwen2.5~\citep{qwen25} (3B/7B/14B),
Qwen3~\citep{yang2025qwen3} (4B/8B),
Qwen3.5~\citep{qwen35} (2B/4B/9B),
Llama3.1~\citep{grattafiori2024llama} (8B),
Mistral~\citep{jiang2023mistral7b} (7B),
Phi-4~\citep{abdin2024phi} (14B), and
Gemma-4~\citep{team2026gemma} (12B) instruction-tuned models, with thinking disabled and greedy decoding to reduce sampling variability.
Methods share fitting and evaluation splits within each model and benchmark.
\\
\noindent\textbf{Baselines.}
\emph{Keep all} serves the full registry and \emph{No tools} an empty menu.
\emph{Random} draws $K$ tools.
\underline{Search-Based}:
\emph{BM25}~\citep{robertson2009bm25} ranks tool descriptions by lexical similarity to the fitting requests;
\emph{Dense} uses BGE description embeddings~\citep{xiao2024cpack} and \emph{Tool2Vec}~\citep{tool2vec} generates queries from tool descriptions, encodes them, and aggregates their embeddings into a tool representation.
\underline{LLM-Based}:
\emph{LLM as router}~\citep{shen2023hugginggpt} gives the serving model the request and the name and description of every registry tool, asks it for the list of tools the request needs, and serves exactly that list;
\emph{Jev}~\citep{jevdocs} is an external relevance model that scores every tool of a request with an independent probability from the request and the tool descriptions, and serves the top-$K$ tools.
We additional adapt two output-aware methods:
\emph{TTO}~\citep{octotools} starts from the best single tool on fitting accuracy and adds the tool with the largest gain while accuracy strictly increases. 
\emph{Beam Search}~\citep{lowerre1976harpy} keeps the five best subsets at each size.
TTO and Beam Search need ground-truth labels and is costly in loop tool space composition.
Appendices~\ref{app:conventions} and~\ref{app:tablenotes} give implementation details.

\vspace{-1mm}
\subsection{Effectiveness of \ours{} in Tool Space Optimization}
\label{sec:exp-shared}
\label{sec:exp-tgb}
\label{sec:exp-gta2}
\label{sec:exp-bfcl}
\vspace{-1mm}
Our central claim is that \ours{} constructs a persistent task-specific tool space, which improves task performance.
We evaluate on three benchmarks that cover synthetic compositional tool use, multi-turn function calling, and ReAct execution with real tools.
\\
\textbf{TGB: task-specific spaces remove distractors.} \ours{} improves over the full fifteen-tool menu on all eight models in Table~\ref{tab:tgb-family}.
The gains are largest on models most affected by irrelevant tools: accuracy rises from $0.565$ to $0.815$ on Qwen3-8B and from $0.464$ to $0.658$ on Mistral-7B.
Across task families, \ours{} retains between roughly three and seven tools on average from the original 15 tool registry. Additionally, Table~\ref{tab:kselect} (Appendix~\ref{app:budget}) shows how each family's budget is chosen and which tools it keeps.
Thus, the constructed spaces are the right ones: on Qwen2.5-7B every family's space contains its gold tool chain, and for five of the ten families it is exactly that chain.
\\
\noindent\textbf{BFCL: task-specific spaces benefit multi-turn tool calling.}
BFCL evaluates success over complete multi-turn records, where an incorrect required call can cause the entire record to fail.
Beam search is computationally prohibitive in this setting, so we exclude it from the BFCL experiments.
\ours{} gives the best space on three of the five models and is within 0.21 points of the best, Jev, on the other two: on Qwen3-4B it reaches 28.12\% against 25.62\% for Dense, the strongest baseline there, and on Gemma4-12B it reaches 28.75\% against 28.96\% for Jev. The LLM router is below \ours{} on all five models, by 1.5 to 19.6 points, which reflects the difficulty of selecting tools before a multi-turn interaction unfolds.
These gains of \ours{} therefore reflect improvements in completing multi-turn interactions, beyond isolated tool-call decisions.
\begin{table}[!tb]
\centering
\small
\setlength{\tabcolsep}{3.5pt}
\caption{
\textbf{BFCL v4 \texttt{multi\_turn} accuracy.}
For each task, the first 10 records fit
the space and its remaining 40 records evaluate it (160 records in total). $^\dagger$ marks methods that read correctness labels for all fitting requests. Bold marks the best result in each column; underlining marks the second best. Results are averaged over three seeds.}
\vspace{-2mm}
\label{tab:bfcl-models}
\begin{tabular}{lccccc}
\toprule
Menu & Qwen3-4B & Qwen3.5-4B & Qwen3.5-9B & Gemma4-12B & Phi-4-mini-3.8B \\
\midrule
No tools
& 0.00 & 0.00 & 0.00 & 0.00 & 0.00 \\
Keep all
& $17.71\pm0.95$ & \underline{$14.37\pm0.63$} & $23.12\pm0.63$ & $24.79\pm0.36$ & $3.33\pm0.36$ \\
Random
& $10.21\pm1.06$ & $2.71\pm1.06$ & $13.96\pm0.29$ & $13.54\pm0.29$ & $3.54\pm0.29$ \\
\midrule
\multicolumn{2}{@{}l}{\textit{Search-Based}}  &  &  &  & \\
BM25
& $18.13\pm0.51$ & $12.08\pm0.29$ & $27.08\pm0.78$ & $25.83\pm1.93$ & $3.33\pm0.29$ \\
Dense
& \underline{$25.62\pm1.77$} & $13.96\pm0.29$ & \underline{$27.71\pm0.29$} & $26.04\pm0.29$ & $3.75\pm0.00$ \\
Tool2Vec 
& $24.58\pm0.95$ & $10.42\pm0.95$ & $23.12\pm0.36$ & $23.54\pm1.77$ & $3.54\pm0.36$ \\
\midrule
\multicolumn{2}{@{}l}{\textit{LLM-Based}}  &  &  &  & \\
LLM as router
& $12.71\pm1.18$ & $5.83\pm0.00$ & $9.79\pm0.29$ & $21.25\pm0.88$ & $3.54\pm0.29$ \\
Jev
& $22.50\pm0.00$ & $8.96\pm0.36$ & $22.50\pm0.62$ & {\boldmath$28.96\pm1.57$} & {\boldmath$5.21\pm0.36$} \\
\midrule
TTO$^\dagger$
& {$20.83\pm0.36$} & $12.29\pm1.57$ & {$22.79\pm0.72$} & $24.17\pm0.36$ & $4.79\pm0.36$ \\
\hdashline
\ours{}
&  {\boldmath$28.12\pm 0.72$} & {\boldmath$15.62\pm 0.95$} & {\boldmath$29.38\pm 0.62$}
& \underline{$28.75\pm1.25$} & \underline{$5.00\pm0.62$} \\
\bottomrule
\end{tabular}
\vspace{-4mm}
\end{table} 
\begin{figure}[!b]
\vspace{-6mm}
\centering
\includegraphics[width=\linewidth]{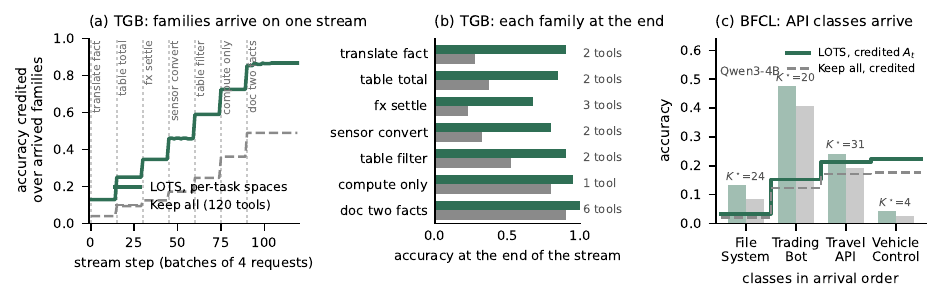}
\vspace{-9mm}
\caption{Accumulating task-specific spaces as tasks arrive. 
(a) Seven of the eight TGB families in this stream arrive on one request stream (Qwen2.5-7B, registry of 120 tools); a step is a batch of four requests, and dotted lines mark arrivals. The curve is credited accuracy $\mathrm{Acc}_t$, it rises in a jump when a family's space is added.
(b) Final accuracy and space size for each family.
(c) Four BFCL API classes arrive in sequence (Qwen3-4B). Bars give each class's accuracy on its own evaluation records; lines give $\mathrm{Acc}_t$ over the four classes, which ends at the mean of the bars.}
\label{fig:stream}
\vspace{-5mm}
\end{figure}
\\
\textbf{GTA-Atomic: task-specific spaces improve later execution by a ReAct agent.} We finally evaluate fitted spaces inside the ReAct execution loop of GTA-Atomic, where the agent calls executable tools over several turns.
For example, a correct output requires reading text from an image and computing with it. The constructed tool space is OCR together with a computation tool (Calculator or Solver).
\ours{} improves over Keep all on six of the seven models, with the largest gains on Qwen3-8B, from $10.71$ to $22.32$, and Qwen2.5-7B, from $7.14$ to $14.29$, where it ties Beam Search; Qwen2.5-14B is the one model where Keep all stays ahead ($8.04$ against $7.14$).
Appendix~\ref{app:interactions} enumerates all tool spaces of a six-tool universe to locate the \ours{} pair among them.
\begin{table}[b]
\vspace{-3mm}
\centering
\footnotesize
\setlength{\tabcolsep}{2pt}

\caption{
\textbf{GTA-Atomic accuracy (\%)}.
$^\dagger$ marks methods that read correctness labels on the fitting requests.
Bold denotes the best result in each column;
underlining denotes the second best.
}
\label{tab:gta-combined}
\vspace{-2mm}
\begin{tabular}{lccccccc}
\toprule
Method
& Qwen2.5-3B & Qwen2.5-7B & Qwen3-8B
& Qwen2.5-14B & Qwen3.5-9B & Phi-4-14B & G4-12B \\
\midrule
No tools
& 1.79 & 2.68 & 4.46 & 4.46 & 6.25 & 4.46 & 2.68 \\
Keep all
& \underline{7.14} & 7.14 & 10.71 & \textbf{8.04} & \underline{40.18} & 11.61 & 24.11 \\
Random
& 6.25 & 5.36 & 11.61 & 6.25 & 16.96 & 8.93 & 14.29 \\
\midrule
\multicolumn{2}{@{}l}{\textit{Search-Based}}  &  &  &  & \\
BM25
& 1.79 & 2.68 & 6.25 & 3.57 & 14.29 & 2.68 & 14.29 \\
Dense
& 0.00 & 1.79 & 1.79 & 1.79 & 12.50 & 0.89 & 5.36 \\
Tool2Vec
& 0.00 & 0.00 & 6.25 & 3.57 & 16.96 & 0.00 & 11.61 \\
\midrule
\multicolumn{2}{@{}l}{\textit{LLM-Based}}  &  &  &  & \\
LLM as router
& \underline{7.14} & 1.79 & 15.18 & 4.46 & 25.00 & 8.04 & 25.00 \\
Jev
& 5.36 & 9.82 & 12.50 & 6.25 & 35.71 & 8.04 & 24.11 \\
\midrule
Beam Search$^\dagger$ & \underline{7.14} & \textbf{14.29} & \underline{18.75} & 4.46 & 33.93 & 14.29 & 25.89 \\
TTO$^\dagger$
& 6.25 & \underline{12.50} & 16.96 & 2.68 & 33.93 & \underline{15.18} & \underline{27.68} \\
\hdashline
\ours{}
& \textbf{13.39} & \textbf{14.29} & \textbf{22.32} & \underline{7.14} & \textbf{45.54} & \textbf{16.07} & \textbf{29.46} \\
\bottomrule
\end{tabular}
\end{table}

Additionally, Appendix~\ref{sec:exp-why} shows why output-based, full-context leave-one-out scoring identifies these spaces while description-based and Shapley-thresholded selection do not.

\subsection{Persistent Task-Specific Evolution as Tasks Arrive}
\label{sec:exp-loop}
\begin{figure}[!t]
\centering
\includegraphics[width=\linewidth]{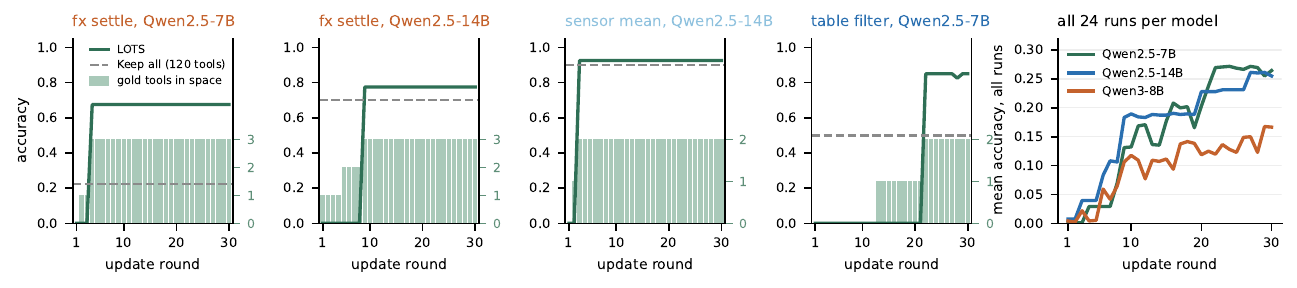}
\vspace{-8mm}
\caption{Improving an existing task space through successive request batches.
TGB, 120-tool registry.
Left: four family updates showing accuracy on the task's fixed evaluation requests.
Right: mean accuracy after each update over all 24 runs per model.}
\label{fig:updates}
\vspace{-7mm}
\end{figure}
Across tasks, we ask whether \ours{} can construct and retain useful task-specific spaces as new task types arrive (Eq.~\ref{eq:continual_space}).
Within a task, we ask whether later output experience of the same task can improve its stored space, $E_t^{(k)}\to E_t^{(k+1)}$ (Eq.~\ref{eq:within_task}).
\\
\textbf{Across tasks: constructing and retaining separate spaces.}
When a new task $c_t$ arrives, \ours{} constructs its space $E_t$ from that task's output traces.
The deployed agent accumulates a collection of reusable configurations, with each task's experience updating only its own space.
Figure~\ref{fig:stream}(a--b) follows seven of the eight families in this stream.
At each step, the active family with the fewest completed rounds receives a batch of four requests.
After $t$ task types have arrived, we report credited accuracy
$
A_t=\frac{1}{T}\sum_{j=1}^{t}
\mathrm{Acc}(c_j;p_\theta,E_j),
$
where $T$ is the total number of task types and each $E_j$ is the corresponding task's current space.
This metric measures performance contributed by the tasks encountered so far.
\ours{} remains above Keep all throughout the stream and finishes at $0.87$ versus $0.49$.
Table~\ref{tab:stream} in Appendix~\ref{app:evolution}
reports the full eight-family stream on seven models and three seeds;
Figure~\ref{fig:stream}(c) also shows the BFCL setting:
each arriving API class fits a documentation configuration on its first records and reuses it on subsequent records.
On Qwen3-4B, all four classes improve over Keep all, with final credited accuracy of $22.3\%$ versus $17.7\%$.
\\
\textbf{Within a task: improving a space with subsequent experience.}
We next fix the task type and examine how its space changes across successive requests.
We limit initial tool exposure: the registry contains 120 tools, but each request receives at most 16.
The initial space may therefore lack tools needed by the task.
Figure~\ref{fig:updates} shows that later experience can improve spaces that are incomplete.
The sharp increases occur when an update adds a missing tool and completes a required tool chain, enabling the agent to solve multiple evaluation requests that previously lacked the necessary tools.
Across all 24 runs per model, mean accuracy also increases over successive updates.
Together, the two studies demonstrate how \ours{} accumulates spaces across tasks and updates individual spaces.

\vspace{-1mm}
\subsection{Serving Efficiency and Cross-Model Reuse}
\label{sec:exp-eff}
\label{sec:exp-cost}
\label{sec:exp-models}
\vspace{-1mm}
In this section, we examine the resulting token savings and whether a space fitted with one model can be reused by another.
\\
\textbf{Construction and serving cost.}
Once constructed, a task-specific space serves subsequent requests
without repeating tool search or ranking over the global registry.
Figure~\ref{fig:cost} reports the combined token cost of tool selection and serving the evaluation requests on GTA-Atomic \texttt{read\_arith}.
Under this accounting, \ours{} uses
$2.0$-$4.0\times$ fewer tokens than TTO and $4.4$-$9.5\times$ fewer than Beam Search.
On TGB, task-specific pruning reduces serving prompt tokens by 33-46\% relative to the full fifteen-tool menu (Appendix~\ref{app:efficiency}).
Figure~\ref{fig:gta-frontier2} further compares accuracy with
serving prompt tokens on GTA-Atomic \texttt{read\_arith}, showing where
smaller spaces improve accuracy and efficiency.
Appendix~\ref{app:efficiency} details the cost and the number of requests needed to recover the construction cost.
\\
\textbf{Reuse across models.}
We next test whether changing the serving model requires
refitting the tool space.
On BFCL, we fit task-specific documentation rankings from
a source model's traces and use them with a different serving model.
Transferred rankings outperform the serving model's
Keep all baseline in 18 of the 20 cross-model pairs in
Figure~\ref{fig:xfer-matrix}.
This suggests that the rankings capture task-level preferences
that can remain useful across models.
Reusing an existing ranking avoids collecting and scoring a new set of traces from the serving model.
Since likelihood access is needed only at fitting time, a space fitted with an open-weight model could also serve closed API models, a direct extension of this reuse.

\begin{figure}[!tb]
\centering
\begin{minipage}[t]{0.48\textwidth}
\centering
\includegraphics[width=\linewidth]{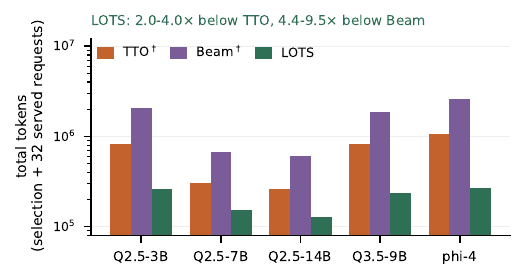}
\vspace{-4mm}
\captionof{figure}{Total token cost on GTA-Atomic \texttt{read\_arith}: tokens spent on tool selection plus the served evaluation requests, log scale. $\dagger$ methods read correctness labels.}
\label{fig:cost}
\end{minipage}\hfill
\begin{minipage}[t]{0.48\textwidth}
\centering
\includegraphics[width=\linewidth]{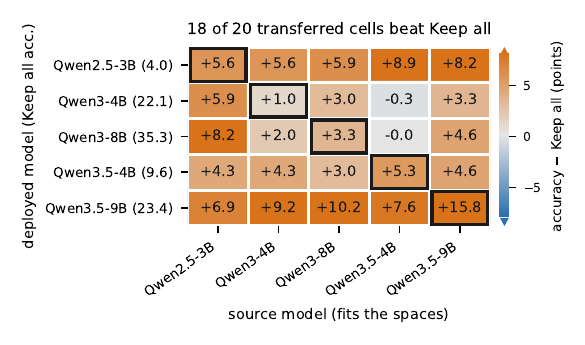}
\vspace{-6mm}
\captionof{figure}{Cross-model tool-space reuse on BFCL.
Rows: execution models; columns: sources; cells: accuracy relative to each row's Keep all baseline; outlined cells: own rankings.}
\label{fig:xfer-matrix}
\end{minipage}
\vspace{-6mm}
\end{figure}

\begin{figure}[!tb]
\centering
\includegraphics[width=\linewidth]{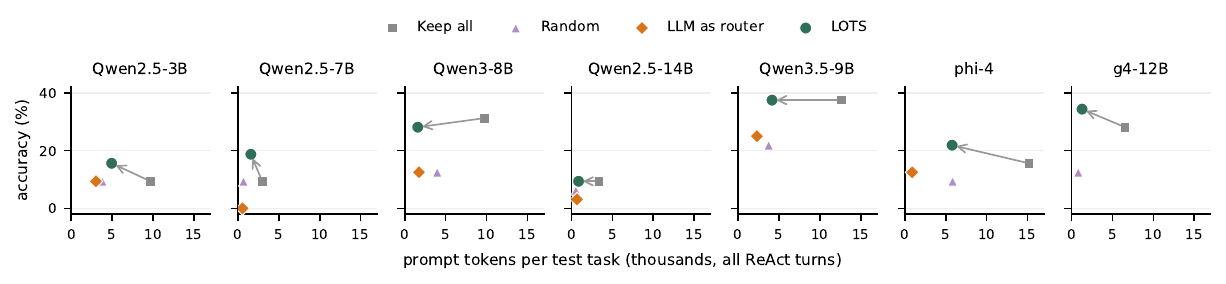}
\vspace{-6mm}
\caption{\textbf{Accuracy against serving cost on GTA-Atomic \texttt{read\_arith}.} Prompt tokens per test task over all ReAct turns against accuracy.
The arrow runs from Keep all to \ours{}.}
\label{fig:gta-frontier2}
\vspace{-5mm}
\end{figure}

\section{Conclusion}
\label{sec:discussion}
\vspace{-1mm}
We introduced \ours{}, a likelihood-guided approach to tool-space self-evolution that turns an agent's own execution experience into persistent, task-specific tool spaces. By measuring how observed tool outputs support a fixed realized answer, \ours{} converts request-level traces into reusable tool configurations that can be refined as new experience arrives. Across compositional tool use, multi-turn function calling, and ReAct execution, these learned spaces consistently reduce unnecessary tool context while improving or preserving task performance. Our results show that an agent's executable environment itself can serve as persistent adaptive state. We provide additional discussion, including related work (App.~\ref{sec:related}), in the Appendix.

\subsection*{AI use statement}
AI assistants were used to support manuscript editing and revision.
The authors reviewed the resulting text and take full responsibility
for the content, experimental results, and conclusions of this paper.

\subsection*{Ethics statement}
This work studies task-specific tool-space optimization for language-model
agents using existing benchmarks and synthetic tasks.
Our experiments evaluate task performance and tool-use efficiency;
they do not establish the safety of autonomous execution.
Likelihood-based tool scores measure support for a model's generated
answer and do not certify its correctness or the safety of a tool.
Practical execution therefore requires appropriate tool permissions
and oversight.

\subsection*{Reproducibility statement}
The main text and appendices describe the task construction,
fitting and evaluation protocols, candidate tool libraries,
model configurations, scoring interventions, tool-budget selection,
and baseline implementations.
The appendices additionally report fixed-budget comparisons,
scoring ablations, and settings for repeated updates and
cross-model tool-space reuse.
We plan to release the implementation, configuration files,
and evaluation scripts, together with request identifiers and
split definitions where permitted by the source datasets.
\bibliographystyle{iclr2027_conference}
\bibliography{iclr2027_conference}

\appendix
\clearpage

\renewcommand{\contentsname}{Table of Contents}
\addtocontents{toc}{\protect\setcounter{tocdepth}{2}}
{\hypersetup{pdfborder={0 0 0}}\color{blue}\tableofcontents}

\newpage
\vspace{-1mm}
\section{Related Work}
\label{sec:related}
\vspace{-1mm}
\textbf{Tool selection for LLM agents.} Tool-augmented language models learn when and what to call from self-supervised signal \citep{toolformer}, and scale to large API inventories with a per-request retriever trained on labeled query--API pairs \citep{gorilla,qin2024toolllm}; later retrievers target complete tool sets \citep{colt,hyset}, hierarchical libraries \citep{toolrerank,anytool}, usage-derived embeddings \citep{tool2vec} and iterative requests \citep{mcpzero}, and a parallel line rewrites documentation once before execution \citep{easytool,draft}.
Frontier harnesses select in two stages, a coarse hand-picked split into a direct set and a deferred pool, then per-request BM25 or regex ranking over the deferred pool \citep{codex_toolsearch,claude_toolsearch}; replacing annotated sets with retrieved ones lowers agent pass rates, strong retrievers still miss target tools \citep{toolret,planbenchxl}, and lexical search misses tools whose names overlap.
Selection in these systems is a serving-time decision for one request over a library kept fixed by hand; where a set is chosen offline, as in the task-specific toolset optimization of \citet{octotools}, the search reads graded rollouts.
AutoTool trains the policy to select tools from an evolving inventory during reasoning \citep{autotool}, and ToolSelf lets the agent reconfigure its own toolbox, strategy and context at run time through tool calls \citep{zhou2026toolself}; both learn a selection or configuration policy from trajectories, whereas \ours{} keeps the model fixed and stores the decision as a ranking over the existing registry.
MetaTool trains transferable tool-use ability into the model through meta-task augmentation \citep{metatool}; the ranking \ours{} stores moves between models with no training, which is what the cross-model reuse of \S\ref{sec:exp-eff} tests.
\ours{} decides after the tools have run, from the trace itself, and keeps the decision as state that later requests share; its comparison against self-report, documentation-based and gradient-based selection is in \S\ref{sec:exp-shared} and Appendix~\ref{sec:exp-why}.
A coarse per-request retriever and a persistent task-specific configuration answer different questions, which candidates a request may see and which tools the task has earned, and could be composed (\S\ref{sec:discussion}).
The premise that a smaller menu can beat the full one is consistent with models using long contexts unevenly \citep{lostinmiddle}.

\textbf{Attribution and valuation.} Context attribution asks which parts of a context a model's output relies on; ContextCite ablates context segments and fits a surrogate to the change in the output's likelihood \citep{contextcite}. \ours{} uses the same kind of ablation signal on tool observations and differs in what it does with it: it aggregates the per-request attributions over the requests of a task and stores the result as a persistent tool space that later requests are served with. Asking what a component is worth by removing it has two standard answers: the Shapley value, which averages a component's marginal contribution over every coalition of the others \citep{shapley1953}, and leave-one-out against the full set; data valuation uses both to price training examples \citep{datashapley}, choosing on grounds of cost.
Here the choice is made on grounds of what the tasks require.
On chained tasks a tool has no value apart from its partners, so a coalition-averaged value paired with an independent threshold reports a small marginal for a load-bearing tool, and a standalone score has nothing to rank; Appendix~\ref{sec:exp-why} measures both with exact Shapley values and finds the highly aligned signal, with aggregation explaining the difference.
\ours{} removes one tool at a time from the full coalition, so each tool is valued with its partners present; and because it scores the likelihood of the agent's own fixed answer rather than its correctness, the resulting ranking requires no labels. Labels are used only to choose the scalar budget K.
Gradient attribution over the tool observations is cheaper, one backward pass, and is compared there too; attention-based readings are not.
Trajectory-grounded credit assignment also drives EvoTool, which localizes a failure to one module of a tool-use policy and mutates that module through natural-language critique \citep{evotool}; \ours{} assigns credit to tools instead, by an intervention on the likelihood of a fixed answer under a fixed model, and needs neither a failure signal nor a mutation step.

\textbf{Self-evolving agents.} The survey \citet{selfevolving2025} defines a self-evolving agent as one that modifies its parameters, context, toolset or architecture from its own trajectories or feedback, and it already covers tools, from tool creation and refinement to tool management and selection; \citet{selfimprove2026} organise the same field by model versus scaffold.
Concrete systems differ in the object experience updates.
Dynamic Cheatsheet \citep{dynamiccheatsheet} keeps a curated memory of strategies and snippets: each query is answered with the memory built from the queries before it, the memory is updated after the answer.
ACE \citep{ace} keeps a context playbook of bullets that a generator, reflector and curator grow by incremental deltas; in its online setting each test sample is predicted under the current context and the context is then updated from that sample, with a batch size of one.
SkillOpt trains a natural-language skill document as the external state of a frozen agent and accepts an edit only when it improves a held-out score \citep{skillopt}.
These systems store experience as text the model reads; \ours{} stores it as the executable tool space the model is served.
Voyager~\citep{voyager} keeps an ever-growing library of executable skills retrieved for later tasks; 
TroVE  \citep{trove} grows a toolbox of reusable functions while solving a stream of tasks and periodically trims it; 
LATM has a tool-maker model write tools that a tool-user model applies, with a cache that stores the functionality of a class of requests for reuse \citep{latm}.
EvoSOP synthesizes standard operating procedures from trajectories into higher-order tools and iterates their construction, merging, evaluation and pruning \citep{evosop}, and SMITH trains one policy to both create tools and use them \citep{smith}; these methods grow the tool pool itself, whereas \ours{} ranks and prunes a fixed registry and could rank the tools such methods create.
SEARL jointly optimizes the policy and a tool-graph memory with trajectory- and step-level credit \citep{searl}, and Tool-R0 co-evolves a task generator and a solver by self-play reinforcement learning from zero data \citep{toolr0}; both change the model's weights under a reward, whereas \ours{} keeps $p_\theta$ fixed and needs no reward, so a space fitted by \ours{} can also serve a policy trained this way.
Adapting under a changing request distribution and the forgetting it induces are studied for model parameters in continual test-time adaptation \citep{cotta}, and continual-learning protocols for tool use \citep{wang2024llms}.
However, none of these methods builds output-aware task-specific tool spaces that evolve from experience traces.

\section{Implementation and Experimental Setup}
\label{app:method-details}
\label{app:protocols}
\label{app:audit}

\ours{} constructs a persistent tool space for each recurring task from the traces of requests served for that task. We first describe the benchmarks and fitting protocols, then give the scoring and update implementations used in the experiments. Throughout, $p_\theta$ is the frozen model, $x$ is a request, $y_x$ is its recorded output, and $h_x$ is the context used to score that output. 
The global tool registry (space) is $\mathcal{S}$, and $E_t\subseteq\mathcal{S}$ denotes the tool subset maintained for task $c_t$.
Figure~\ref{fig:overview} in the main text summarizes the procedure.

\subsection{Benchmarks and Task Definitions}
\label{app:datasets}
\label{app:benchmark-settings}

\textbf{TGB.}
TGB is a synthetic, text-only benchmark for compositional tool use, with 4,000 requests across ten families. Each request is generated from a hidden scene (a table, sensor series, calendar, or small document corpus) that only the tools can access. The question requires the family's tool chain, and the gold answer is computed from the scene. Upstream tools return multiple records, such as six table rows, five readings, or two passages, so the downstream computation requires combining their outputs. The \texttt{compute\_only} and \texttt{no\_tool} families provide tool-optional and tool-free controls (Table~\ref{tab:tgb-families}).
The registry contains seven chain tools and eight distractors: UnitConvert, TempConvert, CurrencyConvert, DurationCalc, Solver, GoogleSearch, Summarize, and Barcode. Distractors are executable; on an unsuitable input, each applies its own transformation and returns a result in its usual format. Larger registries add synthetic distractors with seeded names such as \texttt{TaxCalc} and \texttt{PayrollForecast}, each returning a plausible number derived from the scene. Accuracy measures whether the final answer matches the gold value.

\begin{table}[H]
\centering\footnotesize
\caption{TGB families and their tool chains.}
\label{tab:tgb-families}
\fitwidth{%
\begin{tabular}{@{}lll@{}}
\toprule
Family & Chain & Required computation \\
\midrule
table\_total & TableQuery $\to$ Calculator & sum of six quantity$\times$price products \\
table\_filter & TableQuery $\to$ Calculator & the same sum over a filtered subset \\
sensor\_mean & SensorAPI $\to$ Calculator & mean of five readings \\
sensor\_convert & SensorAPI $\to$ Calculator & mean, then a unit conversion \\
schedule\_gap & CalendarAPI $\to$ Calculator & minutes between two clock times \\
doc\_two\_facts & DocRetrieve $\to$ Calculator & two retrieved prices combined \\
translate\_fact & DocRetrieve $\to$ Translate & a figure written in cipher words \\
fx\_settle & DocRetrieve $+$ ExchangeRate $\to$ Calculator & price $\times$ quantity $\times$ rate \\
compute\_only & Calculator (optional) & tool-optional control \\
no\_tool & none & tool-free control \\
\bottomrule
\end{tabular}}
\end{table}

\textbf{BFCL.}
The primary BFCL evaluation uses the 200 records of BFCL v4 \texttt{multi\_turn\_base} \citep{bfcl}, grouped by primary API class: GorillaFileSystem, TradingBot, TravelAPI, and VehicleControlAPI, with 50 records per class. A record succeeds only when the required calls across all turns reach the reference end state. We therefore use documentation demotion rather than tool removal: all tools remain callable, including those needed only in a later turn. The cross-model and budget diagnostics use a separate 101-record tune set, with fitting and evaluation on the same records.

\textbf{GTA-Atomic.}
GTA-Atomic, the atomic-task split of GTA \citep{gta}, contains 229 image-based requests served by a ReAct agent through the official harness and an executable tool server. Four annotated numeric task types provide exactly scoreable requests: \texttt{read\_arith} (65), \texttt{web\_fact} (35), \texttt{count\_arith} (30), and \texttt{math\_eq} (17). The main experiments (Table~\ref{tab:gta-combined}) use the three types \texttt{read\_arith}, \texttt{count\_arith} and \texttt{web\_fact}; \texttt{math\_eq} is not used. The Keep-all menu lists the 14 tools of the tool server, and the schema-token counts are taken over these 14. Fitting uses the seven candidates whose outputs can be cached for each request: OCR, ImageDescription, Calculator, CountGivenObject, RegionAttributeDescription, TextToBbox, and Solver. Accuracy is exact match between the final answer and the reference.
Task assignments come from benchmark metadata: TGB families, BFCL API classes in the primary experiments, and annotated GTA-Atomic types.

\textbf{Baselines.}
\label{app:conventions}
Unless otherwise stated, task-level baselines construct one tool
space per task using its fitting requests and reuse that space
during evaluation.
The two groups in Tables 1–3 are separated by what selection consults: the search-based group ranks tools by matching requests to tool descriptions, whereas the LLM-based group runs the serving model during selection. In the second group, the LLM router and Jev select per request; TTO and Beam Search execute candidate spaces on the fitting requests, select once per task, and reuse the selected space for evaluation.
Under pruning, the selected tools form the callable menu.
Under documentation demotion, all tools remain callable, but only
the selected tools retain their full documentation.
Budget-matched baselines use the same task-specific budget
$K$ as \ours{}.
Gold answers or gold tool annotations are used only by the
explicitly labeled references and accuracy-based search methods
described below; evaluation labels are not used for selection.
\\
\vspace{-4mm}
\begin{itemize}[leftmargin=1.4em,itemsep=3pt,topsep=4pt,parsep=0pt]
    \item \textbf{Keep all.}
    Exposes the full benchmark-specific tool registry with
    complete documentation.

    \item \textbf{No tools.}
    Exposes an empty callable tool set.

    \item \textbf{Random.}
    Selects a uniformly sampled tool subset with the same
    size as the \ours{} space.
    The subset is reused across evaluation requests of the
    same task.
    Experiments using a fixed budget or repeated random draws
    specify these settings separately.

    \item \textbf{Jev.}
    An external relevance model (TypeSafe System One, \texttt{jev-1.13.0})
    reads each evaluation request and the tool descriptions and answers
    one yes/no question per tool with an independent probability; the
    $K$ highest-scoring tools serve that request. It selects per
    request, as the LLM router does, and reads no trace. Jev is a
    commercial closed model accessed through an API; we pin the model
    version, cache every call by the hash of its input, and release the
    per-tool scores and the resulting spaces, so the evaluation can be
    rerun without API access. It receives the same request text and tool
    descriptions as the other description-based baselines.

    \item \textbf{Oracle.}
    Selects tools using gold tool annotations.
    This baseline serves as a labeled reference rather than
    a deployable selection method; its performance is not
    assumed to be an upper bound.

    \item \textbf{BM25.}
    Scores tools by lexical matching between requests and
    tool names, descriptions, and schemas, without observing
    tool outputs.
    Experiments using per-request retrieval are identified
    separately.

    \item \textbf{Dense.}
    Uses BGE embeddings to compare request text with tool
    descriptions.
    Tools are ranked by their embedding similarity
    to the task's fitting requests, and the top-$K$
    tools are retained.
    This baseline does not observe execution traces or
    tool outputs.

    \item \textbf{LLM router.}
    Prompts the serving model to select tools from their
    descriptions once per request, before observing tool
    outputs.
    The selected menu is used for that request.
    Unlike task-level baselines, the router repeats selection
    for each evaluation request.

    \item \textbf{Tool2Vec}~\citep{tool2vec}.
    Our TGB implementation uses Qwen2.5-7B-Instruct to generate
    20 queries per tool from its name and description,
    embeds them with e5-base-v2, and averages their embeddings
    into a tool vector.
    A separate request-level variant with a fixed four-tool
    budget reaches 0.359 accuracy on Qwen2.5-7B.

    \item \textbf{Trace judge.}
    Gives an LLM judge the executed trace and generated answer
    and asks which tool outputs the answer depends on.
    The resulting judgments are aggregated over fitting
    requests to rank tools, and the highest-ranked tools
    are retained at the comparison budget.
    This baseline observes tool outputs, like \ours{}, but
    attributes their contribution through explicit LLM
    judgments rather than likelihood differences.
    Gold answers are not provided to the judge.

    \item \textbf{TTO.}
    Searches for a task-specific tool subset using
    fitting-set accuracy as the selection objective.
    Candidate spaces are evaluated through additional
    execution on fitting requests, and the selected space
    is reused during evaluation.
    This is a label-using baseline: it requires reference
    answers to evaluate candidate spaces.

    \item \textbf{Beam Search.}
    Searches over tool subsets while retaining multiple
    candidate spaces at each search step.
    Candidates are ranked by fitting-set accuracy, and the
    selected space is reused during evaluation.
    Like TTO, this baseline uses fitting labels and additional
    execution to compare candidate spaces. The search starts from the
    empty space and stops when the best candidate of the next size does
    not strictly improve fitting accuracy. On TGB chain families every
    single tool scores zero on the fitting requests, because the answer
    needs two or three tools together, so no one-tool space improves on
    the empty space and the search stops there.
\end{itemize}

The labeled search baselines evaluate candidate menus on fitting accuracy. 
\emph{TTO} adds the tool giving the largest improvement while accuracy strictly increases; 
\emph{beam search} retains five candidate sets per expansion. These methods are marked $\dagger$. 
The TGB inclusion comparator evaluates each one-tool addition to an empty menu and retains tools with a positive accuracy change.
We use ``label-free'' only for operations with no correctness-label input: this includes own-output scoring.

\subsection{Computing Tool Scores}
\label{app:subspaces}
\label{app:score-details}
\label{app:info}
\label{app:scoring-impl}
\label{sec:span_lots}
\label{sec:span_lots_setup}
\label{app:implementation}

For a recorded output $y_x$, \ours{} compares its likelihood under the original scoring context $h_x$ and a context $h_x^{-s}$ in which tool $s$'s evidence is removed:
\begin{equation}
I_x(s)=\mathcal L_\theta(y_x\mid h_x)-\mathcal L_\theta(y_x\mid h_x^{-s}),
\qquad
\mathcal L_\theta(y\mid h)=\frac{1}{|y|}\sum_{j=1}^{|y|}\log p_\theta(y_j\mid h,y_{<j}).
\label{eq:app-score}
\end{equation}
Both terms teacher-force the same output tokens; the model does not generate a replacement answer. The scored tokens are located within the full tokenized sequence, since separately tokenizing the answer can change its boundary. Prompts use the model's chat template with thinking disabled. Tool documentation remains unchanged during scoring. The intervention on the recorded evidence depends on the benchmark's execution format.

On \textbf{TGB}, the context lists one line per served tool in a fixed canonical order, formatted as \texttt{<Tool>: <output>} and truncated to 1,500 characters per line. Each tool runs once per request. The family's chain executes in dependency order, and other served tools run with generic arguments. Execution is deterministic given the request and menu. Removing $s$ re-executes the chain without it: its line disappears, and a downstream tool that requires its output receives an unresolved argument and returns an error such as \texttt{Calculator: NameError}. Outputs unrelated to $s$ remain unchanged. An upstream tool's score therefore includes the effects propagated through its downstream dependencies. The whole frozen answer is scored.

On \textbf{BFCL}, each record is scored turn by turn from its inference log. The scoring sequence contains the system prompt and query, with the original menu, followed by the turn's recorded steps. The last assistant message of the turn is the frozen target; it is never removed, including when it is itself a call to the candidate tool, so both contexts score the same string. An intervention removes, from the preceding steps, the assistant messages calling the candidate tool and their observations. If an assistant message calls several tools, the implementation removes that whole message and all observations of the step. Scores therefore reflect the removal of the recorded step in these cases. In the fitting traces of the five BFCL models, 24.9\% of tool-calling assistant messages call two or more distinct tools (5.7\% for Phi-4-mini to 30.9\% for Qwen3-4B), and 44.1\% of the (record, called tool) pairs involve at least one such message (12.3\% to 54.0\%), so up to that share of BFCL scores also carries the contribution of a co-called tool. Repeated calls to the candidate tool are removed together. Turn-level differences are averaged within a record; a record with no scorable turn contributes nothing to the aggregate.

On \textbf{GTA-Atomic}, candidate outputs are computed independently, cached once per request, and truncated to 400 characters. An intervention deletes only the candidate's line. The four-type fits score the first number-like expression in the model's final full-menu message, or its last line if no number is present; only empty final messages are excluded. For \texttt{read\_arith}, this yields 13 scored fitting requests for four models, 12 for two models, and 7 for phi-4. The primary 33/32 fit scores the same string. The answer boundary is located using the longest common token prefix of the context alone and the context followed by the answer. The separate 65-request diagnostic scores numeric, percentage, and currency expressions, excluding outputs without such expressions.

The score measures dependence of the recorded output on the intervened evidence, rather than correctness. For fixed contexts, the unnormalized sequence log-ratio $J(y)=\log p_\theta(y\mid h)-\log p_\theta(y\mid h^{-s})$ obeys
\begin{equation}
\mathbb E_{y\sim p_\theta(\cdot\mid h)}[J(y)]
=\mathrm{KL}\!\left(p_\theta(\cdot\mid h)\middle\|p_\theta(\cdot\mid h^{-s})\right).
\label{eq:kl}
\end{equation}
This identity applies to model-sampled sequences and unnormalized log probabilities. \ours{} evaluates a length-normalized difference on a greedily generated output; it is not an unbiased estimate of this KL divergence and need not be nonnegative. Appendix~\ref{app:wrong-answers} examines how incorrect outputs affect the resulting space.

\subsection{Constructing and Updating Tool Spaces}
\label{app:aggregation}
\label{app:update-details}
\label{app:no-backfill}

\textbf{From request scores to a task ranking.}
A request contributes a score $I_x(s)$ only for the tools that were exposed on it and could be scored.
The task score $\bar I_t(s)$ is the mean of these scores over $Q_t(s)$, the task's requests that provide evidence for $s$.
A request on which $s$ was not exposed is left out of this mean; it is not counted as a zero.
Tools are ranked by $\bar I_t(s)$, with ties broken by tool name so that the ranking is deterministic.
The budget $K$ is chosen once per task on its fitting requests: the top-$K$ space is served for every candidate value, and the smallest value with the highest fitting accuracy is kept.
On TGB the candidates are $K=1,\ldots,15$, so the full menu is one of them.
Appendix~\ref{app:budget} lists the chosen budgets.

\textbf{Spaces smaller than the budget.}
Only tools with evidence enter the ranking.
When fewer than $K$ tools have evidence, the space contains exactly those tools, $|E_t|=\min(K,|A_t|)$.
The remaining slots stay empty: a tool that was never scored is never added to reach the budget.
This is why the reported number of tools kept can be below $K$, most visibly in the first batches of a task served under a serving cap.

\textbf{Updating a space across batches.}
Each task keeps a running score $\tilde I_t(s)$ for every tool that has been scored at least once.
After a batch, the batch mean $\bar I_t^{(k)}(s)$ updates the running score by an exponential moving average,
\begin{equation}
\tilde I_t(s)\leftarrow\alpha\,\bar I_t^{(k)}(s)+(1-\alpha)\,\tilde I_t(s),
\label{eq:app-ema}
\end{equation}
and the first batch that scores a tool initializes $\tilde I_t(s)$ to its batch mean.
A tool that the batch did not expose keeps its running score unchanged; the score does not decay.
The next space is the top-$K$ of the running scores over all tools scored so far, so a tool that left the space can return without a change of rule.
The first batch of a task is served under the registry, and later batches under the current space; the running scores of the tools in the space are updated from their traces.
When the registry cannot be served in full (a serving cap $C$, Appendix~\ref{app:within}), the first batch is served under a random sample of $C$ registry tools and each later request is explored with probability $\varepsilon$: it is served under the current space together with a random sample of other registry tools, up to $C$ in total, which is the only way a tool outside the space obtains evidence.
The task-arrival stream of Figure~\ref{fig:stream} uses batches of four requests, $\varepsilon=0.25$, and $\alpha=0.3$.
The cold-start runs of Figure~\ref{fig:updates} use a cap of 16 tools, $\varepsilon=0.5$, and $\alpha=0.3$.

\textbf{Stored state.}
The persistent state of a task is its running scores and the resulting list of tool names; no model parameter or prompt is stored.
A deployment that has served $t$ tasks therefore holds $t$ such lists, and serving a request only reads the list of its task.
Task assignments come from benchmark metadata in our experiments (Appendix~\ref{app:datasets}); without metadata, requests are clustered into tasks, and a request that matches no existing task opens a new space.

\subsection{Scoring Details}
\label{app:tablenotes}
Accuracy is the fraction of evaluation requests answered correctly, with BFCL scored over complete records. 
Tools kept counts callable tools under pruning and fully documented tools under demotion. 
Prompt tokens include the history and documentation processed at every agent turn; 
BFCL sums turns within each record. Schema tokens count only the tool-definition block. Selection cost and break-even counts are defined in Appendix~\ref{app:efficiency}.
The statistical unit is a request or record. Repeated greedy runs measure operational variability, not additional independent samples. Stream bands show ranges across the reported seeds. A 40-request evaluation set changes by 2.5 percentage points per answer; repeated serving of an unchanged space sometimes changes one answer.

\section{Additional Task-Performance and Transfer Results}
\label{app:results}
\label{app:fixed-results}

\subsection{Cross-Model Reuse on BFCL}
\label{app:transfer}
\label{app:bfcl}

Table~\ref{tab:bfcl-xfer-matrix} gives the complete transfer matrix behind Figure~\ref{fig:xfer-matrix}: ten execution models, eleven source models, $K=16$ fully documented tools, three seeds.

\begin{table}[H]
\centering\scriptsize
\setlength{\tabcolsep}{3pt}
\caption{Cross-model configuration reuse on BFCL, $K=16$. Rows are execution models; columns are source models. Underlining marks the own-model configuration and bold the best configuration in each row.}
\label{tab:bfcl-xfer-matrix}
\fitwidth{%
\begin{tabular}{l@{\hskip 4pt}c@{\hskip 6pt}ccccccccccc}
\toprule
Deploy $\downarrow$ & Full & 2.5-3B & 2.5-7B & 2.5-14B & 3-4B & 3-8B & 3.5-2B & 3.5-4B & 3.5-9B & Phi-m & g-12B & g-31B \\
\midrule
Qwen2.5-3B & 4.0 & \underline{9.6} & 11.9 & 11.9 & 9.6 & 9.9 & 7.6 & 12.9 & 12.2 & 11.2 & \textbf{14.8} & 8.6 \\
Qwen2.5-7B & 16.2 & \textbf{18.5} & \underline{15.5} & \textbf{18.5} & 16.5 & 17.5 & \textbf{18.5} & 17.8 & 14.8 & 13.9 & 17.8 & 18.1 \\
Qwen2.5-14B & 18.8 & 18.8 & 19.8 & \textbf{\underline{20.5}} & 15.5 & 14.5 & 17.2 & 15.2 & 13.9 & 17.5 & 18.8 & 14.8 \\
Qwen3-4B & 22.1 & 28.1 & 27.1 & 25.7 & \underline{23.1} & 25.1 & 21.4 & 21.8 & 25.4 & 25.7 & \textbf{30.0} & 23.1 \\
Qwen3-8B & 35.3 & 43.6 & \textbf{44.2} & \textbf{44.2} & 37.3 & \underline{38.6} & 37.0 & 35.3 & 39.9 & 36.3 & 40.9 & 38.0 \\
Qwen3.5-2B & 5.0 & 0.0 & 1.0 & 1.3 & 0.0 & 0.0 & \textbf{\underline{1.6}} & 0.0 & 0.0 & 1.3 & 0.0 & 1.0 \\
Qwen3.5-4B & 9.6 & 13.9 & 14.8 & 9.2 & 13.9 & 12.5 & 11.9 & \underline{14.8} & 14.2 & \textbf{16.2} & 15.8 & 12.5 \\
Qwen3.5-9B & 23.4 & 30.4 & 33.3 & 31.0 & 32.7 & 33.7 & 31.4 & 31.0 & \textbf{\underline{39.3}} & 28.1 & 35.6 & 32.3 \\
Phi-4-mini & 5.0 & 7.9 & 5.6 & 5.0 & 7.3 & 5.9 & 5.6 & 5.0 & 5.9 & \underline{4.3} & \textbf{8.6} & 6.6 \\
gemma-4-12B & 30.0 & \textbf{37.0} & 36.6 & 31.7 & 27.4 & 29.7 & 32.3 & 30.0 & 31.7 & 30.7 & \underline{25.4} & 30.0 \\
\bottomrule
\end{tabular}}
\end{table}

\textbf{Reuse under the primary-class protocol.}
Table~\ref{tab:bfcl-xfer16} repeats the test under the primary-class protocol of Section~\ref{sec:experimental-setup}: each source ranks a class's tools from its own 10 fitting records, each execution model documents the top $K=16$, and accuracy is read on the 160 evaluation records over three seeds. Eight of the nine source--execution pairs improve on Keep all by 1.3 to 6.1 points and the ninth matches it.

\begin{table}[H]
\centering\footnotesize
\caption{Cross-model reuse on BFCL with independent fitting and evaluation records: accuracy (\%) on the 160 evaluation records at a fixed $K=16$, mean over three seeds (Gemma4-12B as execution model: two seeds). Rows are execution models; columns are the source of the ranking.}
\label{tab:bfcl-xfer16}
\begin{tabular}{@{}lcccc@{}}
\toprule
execution model & from Qwen3-4B & from Qwen3.5-9B & from Gemma4-12B & Keep all \\
\midrule
Qwen3-4B & \textbf{21.5} & 18.1 & 19.4 & 18.1 \\
Qwen3.5-9B & 27.7 & 26.5 & \textbf{29.4} & 23.3 \\
Gemma4-12B & 26.9 & \textbf{29.1} & 27.5 & 25.0 \\
\bottomrule
\end{tabular}
\end{table}

\subsection{Tool Budgets and Documentation}
\label{app:budget}
\vspace{-4mm}
Table~\ref{tab:kselect} lists the fitting accuracies that select the budget $K$ on Qwen2.5-7B (selected columns of the $K=1,\ldots,15$ sweep). Compact chains prefer compact spaces (\texttt{fx\_settle}: 0.72 at $K=3$, 0.16 at $K=15$), while both sensor families select large budgets.

\vspace{-4mm}
\begin{table}[H]
\centering\scriptsize\setlength{\tabcolsep}{3pt}
\caption{TGB budget selection on Qwen2.5-7B. Entries are accuracies on each family's 80 fitting requests. The sweep runs over $K=1,\ldots,15$ and the table shows selected columns with the chosen $K$; at most four retained tools are listed.}
\label{tab:kselect}
\fitwidth{%
\begin{tabular}{@{}lcccccccccl@{}}
\toprule
Family & $K=1$ & $K=2$ & $K=3$ & $K=4$ & $K=6$ & $K=8$ & $K=12$ & $K=15$ & Selected $K$ & Tools at selected $K$ \\
\midrule
compute only & 0.94 & 0.85 & 0.82 & 0.80 & 0.80 & 0.82 & 0.81 & 0.78 & 1 & Calculator \\
doc two facts & 0.01 & 0.99 & 0.99 & 0.97 & 0.97 & 0.96 & 0.96 & 0.99 & 2 & DocRetrieve, Calculator \\
fx settle & 0.01 & 0.00 & 0.72 & 0.53 & 0.23 & 0.15 & 0.10 & 0.16 & 3 & DocRetrieve, ExchangeRate, Calculator \\
no tool & 0.99 & 0.99 & 0.99 & 0.99 & 0.99 & 0.99 & 0.99 & 0.99 & 1 & Calculator \\
schedule gap & 0.00 & 0.23 & 0.28 & 0.45 & 0.64 & 0.69 & 0.56 & 0.59 & 9 & CalendarAPI, Calculator, GoogleSearch, Barcode \ldots \\
sensor convert & 0.00 & 0.80 & 0.61 & 0.49 & 0.65 & 0.65 & 0.76 & 0.79 & 14 & SensorAPI, Calculator, GoogleSearch, UnitConvert \ldots \\
sensor mean & 0.00 & 0.34 & 0.36 & 0.69 & 0.54 & 0.65 & 0.55 & 0.56 & 14 & SensorAPI, Calculator, Barcode, UnitConvert \ldots \\
table filter & 0.00 & 0.88 & 0.75 & 0.80 & 0.79 & 0.76 & 0.84 & 0.84 & 2 & TableQuery, Calculator \\
table total & 0.00 & 0.85 & 0.74 & 0.59 & 0.69 & 0.69 & 0.60 & 0.60 & 2 & TableQuery, Calculator \\
translate fact & 0.00 & 0.89 & 0.90 & 0.88 & 0.82 & 0.82 & 0.88 & 0.88 & 3 & DocRetrieve, Translate, UnitConvert \\
\bottomrule
\end{tabular}}
\end{table}

Figure~\ref{fig:bfcl-ksweep} varies the documentation budget on the 101 BFCL tune records. The best budget depends on the source and execution model: Qwen3.5-9B reaches 39.3\% at $K=16$ and 26.1\% at $K=31$, Qwen2.5-7B 15.8\% and 19.8\%.

\begin{figure}[!htb]
\centering
\includegraphics[width=\linewidth]{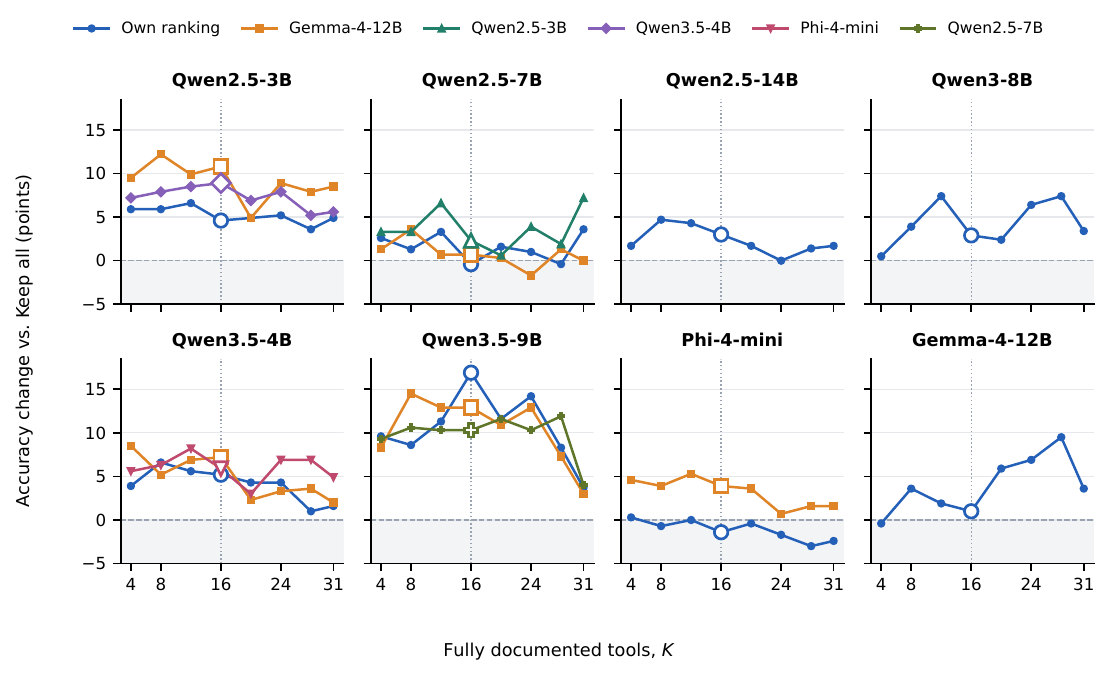}
\caption{BFCL documentation-budget sensitivity on the 101 tune records: accuracy change over Keep all, in percentage points, as the number of fully documented tools $K$ varies.
Each panel is one deployed model; colors give the source of the ranking.
Hollow markers mark $K=16$; at $K=31$ all documentation is kept and only the menu order changes.
Results average three seeds, two for Qwen3-8B.}
\label{fig:bfcl-ksweep}
\end{figure}

Table~\ref{tab:gta-desc} compares pruning with documentation demotion on GTA-Atomic. On Qwen2.5-7B, demotion reaches 18.75\% on the 32-request split against 12.50\% for pruning; pruning cuts the schema from 2,147 to 186 tokens and raises Qwen2.5-7B's total tool calls from 117 to 173.

\begin{table}[H]
\centering\footnotesize
\caption{GTA-Atomic pruning and documentation variants, with one serving session per model. Accuracy (\%) is reported on the 32-request \texttt{read\_arith} split and all 229 requests. Demotion keeps descriptions only for OCR and Calculator. Schema tokens count one tool-definition block; calls are totals over 229 requests.}
\label{tab:gta-desc}
\fitwidth{%
\begin{tabular}{lrrrrr}
\toprule
Menu & Tools & Schema tokens & 32 requests & All 229 & Tool calls \\
\midrule
\multicolumn{6}{l}{\emph{Qwen2.5-7B}} \\
Keep all & 14 & 2,147 & 0.00 & 12.6 & 117 \\
No descriptions & 14 & 1,545 & 12.50 & 12.9 & 123 \\
Demote (\ours{} ranking) & 14 & 1,587 & \textbf{18.75} & 15.7 & 143 \\
\ours{} pruning & 2.0 & 186 & 12.50 & \textbf{16.9} & 173 \\
LLM as router & 1.2 & 128 & 0.00 & 11.3 & 117 \\
\midrule
\multicolumn{6}{l}{\emph{Qwen2.5-14B}} \\
Keep all & 14 & 2,147 & 9.38 & 13.7 & 152 \\
No descriptions & 14 & 1,545 & 6.25 & 11.2 & 194 \\
Demote (\ours{} ranking) & 14 & 1,587 & 6.25 & 14.0 & 166 \\
\ours{} pruning & 2.0 & 186 & 9.38 & 14.9 & 163 \\
LLM as router & 1.6 & 159 & 3.12 & 12.7 & 118 \\
\bottomrule
\end{tabular}}
\end{table}

\FloatBarrier

\subsection{Ranking Quality at a Fixed Budget}
\label{app:fixedk}

Every ranking is deployed at the same $K\in\{2,4\}$ per family on the 320 evaluation requests, which separates the ranking from the label-based choice of the budget. \ours{} uses the ranking behind the main table; Tool2Vec and Dense rank tools by similarity to the family's 80 fitting requests; Random averages five subsets per family.

\begin{figure}[H]
\centering
\includegraphics[width=\linewidth]{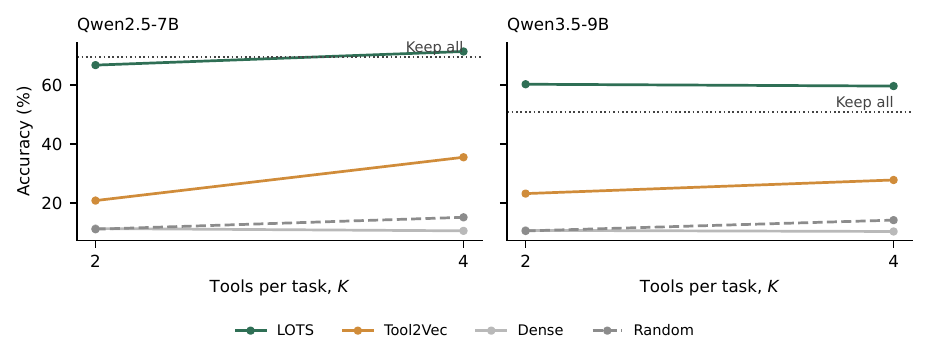}
\caption{TGB accuracy against the budget $K$, every ranking deployed at the same $K$ per family. Macro average over the ten families; Random is the mean over five subsets per family; the dotted line is Keep all.}
\label{fig:tgb-fixedk}
\end{figure}
\vspace{-4mm}

\begin{table}[H]
\centering\small\setlength{\tabcolsep}{4pt}
\caption{Paired accuracy difference, \ours{} minus the baseline, at each fixed budget, with a 95\% interval from a family-stratified paired bootstrap over evaluation requests (2,000 resamples). Positive values favor \ours{}.}
\label{tab:tgb-fixedk}
\begin{tabular}{@{}lcc@{}}
\toprule
Baseline & $K=2$ & $K=4$ \\
\midrule
\multicolumn{3}{@{}l}{\textbf{Qwen2.5-7B}} \\
Tool2Vec & $+46.0$ [$+44.9$, $+47.1$] & $+35.9$ [$+34.6$, $+37.3$] \\
Dense & $+55.5$ [$+54.3$, $+56.8$] & $+60.9$ [$+59.5$, $+62.4$] \\
Random & $+55.7$ [$+54.5$, $+57.0$] & $+56.3$ [$+55.0$, $+57.7$] \\
\midrule
\multicolumn{3}{@{}l}{\textbf{Qwen3.5-9B}} \\
Tool2Vec & $+37.2$ [$+36.2$, $+38.2$] & $+31.9$ [$+30.6$, $+33.2$] \\
Dense & $+49.7$ [$+48.5$, $+50.8$] & $+49.4$ [$+48.1$, $+50.7$] \\
Random & $+49.9$ [$+48.7$, $+51.0$] & $+45.6$ [$+44.2$, $+46.9$] \\
\bottomrule
\end{tabular}
\end{table}

At $K=2$ and $K=4$, \ours{} is 32 to 61 points above Tool2Vec, Dense and Random on both executors, and every interval excludes zero. The fitted budgets of the main table lie in this range: seven of the ten Qwen2.5-7B families have $K\le4$. At a fixed budget no fitting label enters space construction.

\section{Tool-Space Construction and Subsequent Updates}
\label{app:evolution}
\label{app:seq-protocols}

These experiments follow tool spaces from their initial construction to later updates. 
 (0) We first replace the metadata task assignment with request clustering (\S\ref{app:routing}).
 (1) We then examine whether \ours{} can construct and reuse a separate space for each task type as new types arrive.
 (2) We then keep the task distribution fixed and test whether updating an existing space improves on freezing it after the first batch.
 (3) Finally, we restrict the tools visible from the outset and test whether exploration can discover missing tools. 

\subsection{Task Routing by Request Clustering}
\label{app:routing}

The main experiments assign each request to a task with benchmark metadata: TGB families, BFCL API classes and annotated GTA-Atomic types. In deployment the assignment can come from metadata, from clustering request embeddings, or from a task router~\citep{zhang2026agentrouter,liu2026task,zhao2026tcandon}; a request that matches no existing task opens a new task space. This section replaces the metadata on TGB with unsupervised clustering of request embeddings and measures the cost of the replacement.

\textbf{Protocol.} Each request is embedded with a frozen sentence encoder, E5-base-v2 with the \texttt{query:} prefix or BGE-base-en-v1.5. The first fifth of TGB (800 requests) fits the router and the tool spaces; the remaining 3{,}200 requests evaluate them. A cluster replaces a family everywhere the main protocol uses one: contributions are aggregated within a cluster, the cluster's ranking is cut at a budget $K$ selected on its own fitting requests, and each evaluation request is served under the space of the cluster it routes to. We compare the family router (R0) with $k$-means on cosine distance at a fixed $k$ (R1), $k$-means at the $k$ that maximises the silhouette score over $k\in[2,20]$ (R2), and online threshold clustering with running-mean centroids (R3), against two controls: a single cluster that pools every request (C1) and a random partition matched to the cluster-size distribution of R1 at $k{=}10$ (C2). Family names do not occur in the request text, which we verify before clustering. Clustering quality is reported as purity, normalised mutual information and adjusted Rand index against the ten family labels, which the router never reads.

\textbf{Clustering recovers the task partition.} At $k{=}10$, $k$-means recovers the family partition exactly under both encoders, and accuracy equals that of the family router to four decimal places on both models (Table~\ref{tab:routing-clustering}). The silhouette criterion selects $k{=}10$ under BGE without being told the number of families, and $k{=}9$ under E5, which merges two families at a cost below half a point. Metadata task labels are a convenience in this setting, and an unsupervised router supplies the same partition. TGB requests are generated from family-specific templates, so this setting is favourable to clustering; the result shows that \ours{} is insensitive to replacing metadata with an unsupervised partition, not that real request streams separate as cleanly.

\textbf{The gain comes from sharing evidence within a task.} Both controls lose 8 points on Qwen2.5-7B and 8 to 9 points on Qwen3.5-9B. C2 keeps the cluster sizes and the fitted budgets of a correct partition and only reassigns which requests share a space, and it performs at the level of a single pooled cluster. The benefit therefore comes from accumulating contributions over requests that use the same tools, with the smaller menu alone accounting for none of it.

\textbf{Over-splitting a task is cheap; merging tasks is costly.} $k{=}20$ costs $0.03$ points on Qwen2.5-7B, and the threshold router at $\tau{=}0.90$ maintains 54 clusters, more than five per family, while staying within a point of the family router. Merging costs more as the budget tightens: at $k{=}5$, where purity falls to $0.500$, accuracy drops by $1.3$ points at $K$ and by $19.6$ points at a fixed $K{=}2$, since a space of two tools cannot serve two merged families at once. A practical router should therefore lean towards splitting. Thresholds do not transfer across encoders, because cosine similarities are not calibrated between them: under E5 the threshold router collapses to fewer than two clusters at $\tau{=}0.80$ and loses $5.7$ points on Qwen2.5-7B and $7.1$ on Qwen3.5-9B, while BGE keeps separate clusters at the same threshold.

\begin{table}[!tb]
\centering
\small
\setlength{\tabcolsep}{4pt}
\caption{\textbf{Routing by request clustering on TGB.} Clustering quality is measured on the 3{,}200 evaluation requests against the ten family labels; accuracy (\%) is the macro average over families under the per-cluster fitted budget $K$, mean $\pm$ standard deviation over three clustering seeds. Deterministic routers have one seed.}
\label{tab:routing-clustering}
\begin{tabular}{lcccccc}
\toprule
Router & Clusters & Purity & NMI & ARI & Qwen2.5-7B & Qwen3.5-9B \\
\midrule
R0: family label & $10$ & $1.000$ & $1.000$ & $1.000$ & $82.84$ & $65.53$ \\
R1: $k$-means, $k{=}10$, E5 & $10$ & $1.000$ & $1.000$ & $1.000$ & $82.84\pm0.00$ & $65.53\pm0.00$ \\
R1: $k$-means, $k{=}10$, BGE & $10$ & $1.000$ & $1.000$ & $1.000$ & $82.84\pm0.00$ & $65.53\pm0.00$ \\
R2: silhouette $k$, E5 ($k{=}9$) & $9$ & $0.900$ & $0.969$ & $0.898$ & $82.38$ & $65.19$ \\
R2: silhouette $k$, BGE ($k{=}10$) & $10$ & $1.000$ & $1.000$ & $1.000$ & $82.84$ & $65.53$ \\
\midrule
R1: $k$-means, $k{=}5$, E5 & $5$ & $0.500$ & $0.758$ & $0.428$ & $81.53\pm0.71$ & $62.92\pm0.78$ \\
R1: $k$-means, $k{=}20$, E5 & $20$ & $1.000$ & $0.881$ & $0.711$ & $82.81\pm0.12$ & $64.92\pm0.16$ \\
R3: threshold $\tau{=}0.90$, BGE & $54.3$ & $1.000$ & $0.869$ & $0.817$ & $82.00\pm0.27$ & $65.31\pm0.24$ \\
R3: threshold $\tau{=}0.95$, BGE & $384$ & $1.000$ & $0.624$ & $0.325$ & $78.70\pm0.06$ & $60.72\pm0.07$ \\
\midrule
C1: one cluster & $1$ & $0.100$ & $0.000$ & $0.000$ & $74.69$ & $57.72$ \\
C2: random partition & $10$ & $0.128$ & $0.006$ & $0.000$ & $74.72$ & $56.62$ \\
\bottomrule
\end{tabular}
\end{table}
\subsection{Visualization of Task Space}
\label{app:space-viz}
\vspace{-4mm}
Figures~\ref{fig:viz-config}--\ref{fig:viz-flow} show one fitted space per benchmark, the evolution of one space, and the agent's first tool call. In Figure~\ref{fig:viz-config}, the \ours{} space for TGB \texttt{fx\_settle} is exactly the chain (13.1\% to 63.1\%); on BFCL TradingBot it documents in full the twelve functions the agent calls most (40.8\% to 47.5\%); on GTA-Atomic \texttt{web\_fact} with Qwen3.5-9B it keeps six of eight tools (52.9\% to 70.6\%, Random 52.9\%). Figure~\ref{fig:viz-evolution} follows the cold-start run for \texttt{fx\_settle}: DocRetrieve enters after the second batch, ExchangeRate and Calculator after the fourth, and accuracy rises from 0\% to 67.5\% against 22.5\% for the full registry. Figure~\ref{fig:viz-flow} links the first tool call to the outcome on GTA-Atomic: under the \ours{} space the agent opens 12 of 17 requests with ImageDescription, 9 of them correctly, with 2.9 calls per request against 3.9 under the full menu.

\vspace{-4mm}
\begin{figure}[h]
\centering
\includegraphics[width=\linewidth]{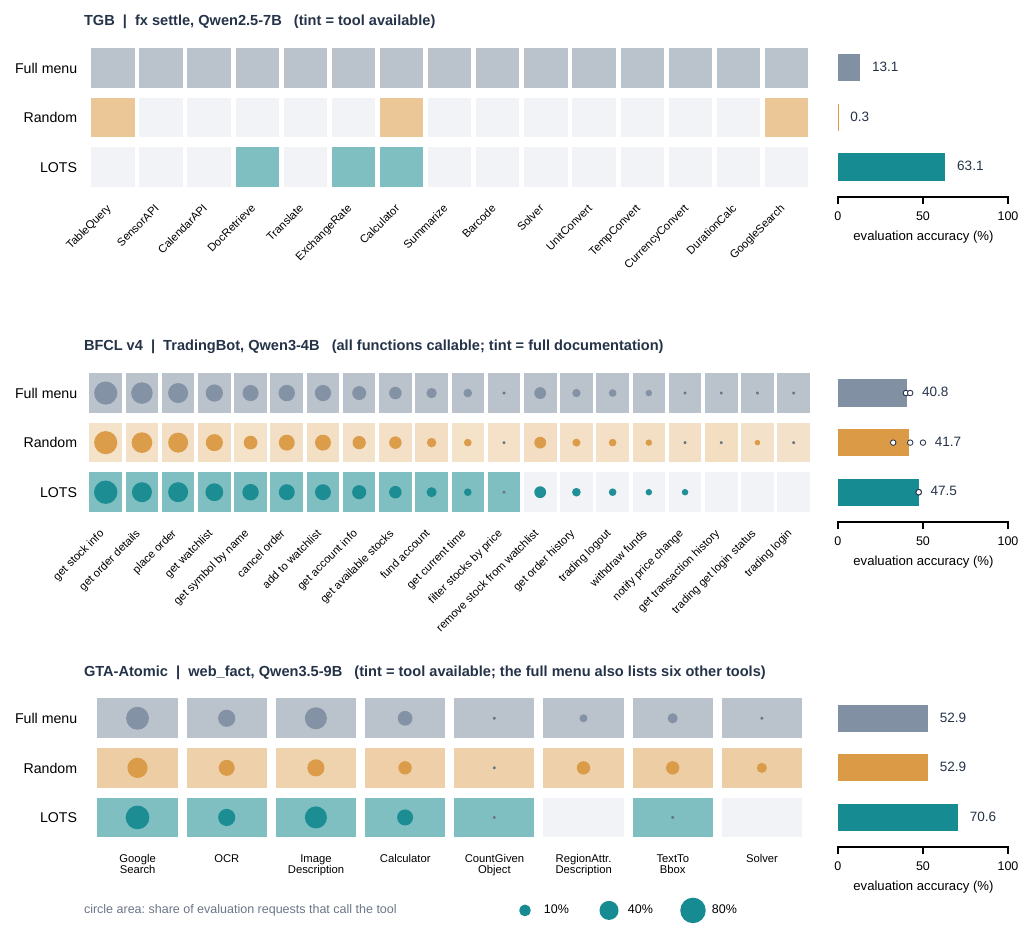}
\caption{Tool space, tool use, and performance for one task per benchmark. Tint marks the tools in the space; on BFCL, where every function stays callable, it marks full documentation, and a paler tint for Random gives the share of records in which the tool is documented. Circle area is the share of evaluation requests that call the tool. TGB executes every served tool once per request, so its panel shows the configuration only. Bars give evaluation accuracy; white dots are the three seeds on BFCL.}
\label{fig:viz-config}
\end{figure}

\begin{figure}[H]
\centering
\includegraphics[width=\linewidth]{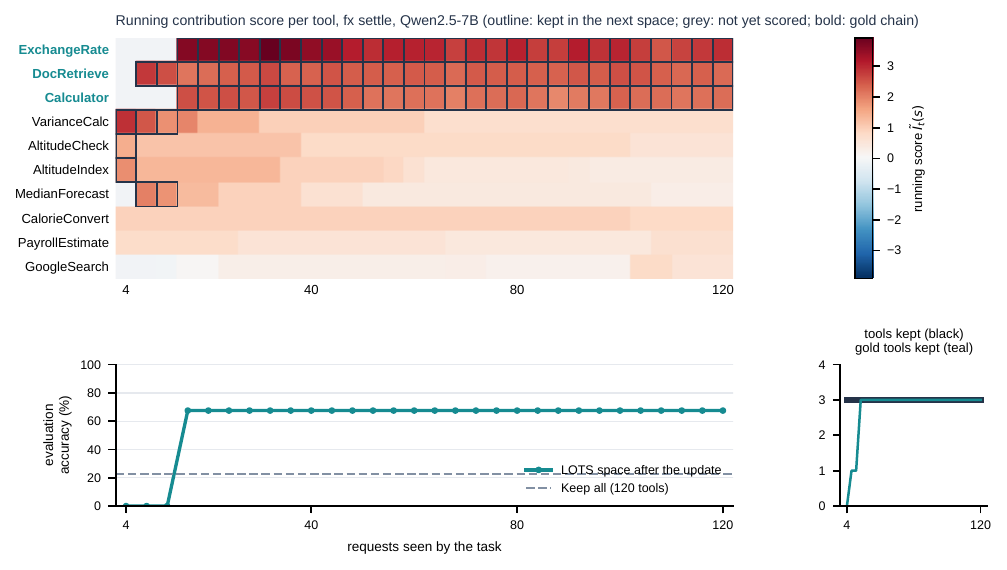}
\caption{Evolution of one persistent space: TGB \texttt{fx\_settle}, Qwen2.5-7B, 120-tool registry with a cap of 16 tools per request, batches of four requests, $\varepsilon=0.5$, $\alpha=0.3$, $K=3$. Top: running score of the ten highest-scoring tools after each batch; outlines mark the tools kept in the next space; bold labels mark the task's chain. Bottom: accuracy of the updated space on the evaluation requests, and the number of tools and chain tools it keeps.}
\label{fig:viz-evolution}
\end{figure}

\begin{figure}[H]
\centering
\includegraphics[width=\linewidth]{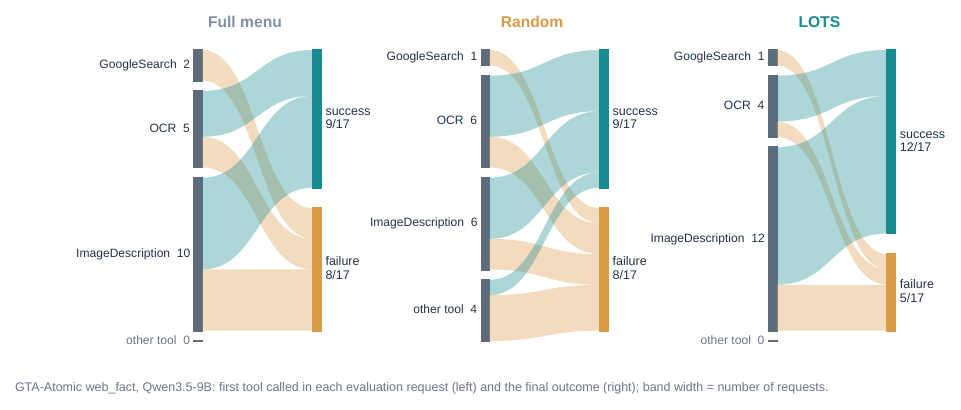}
\caption{First tool call and final outcome on GTA-Atomic \texttt{web\_fact} with the search tool enabled. Each band links the first tool the agent calls on an evaluation request to whether the request is answered correctly; band width is the number of requests. The figure describes association and omits the later calls of a request.}
\label{fig:viz-flow}
\end{figure}

\subsection{Sizes of Task Space}
\label{app:space-size}

Table~\ref{tab:space-size} reports the mean number of tools each method leaves to the agent: callable tools on TGB, fully documented tools on BFCL.

\begin{table}[H]
\centering\footnotesize\setlength{\tabcolsep}{3pt}
\caption{Mean number of tools in the deployed space. TGB counts callable tools out of a 15-tool registry; BFCL counts fully documented tools out of 79 to 105 candidates per API class. Methods marked $^\ast$ use the budget $K$ of \ours{} by construction; methods marked $^\dagger$ read correctness labels on the fitting requests; methods marked $^\ddagger$ select per request, and their count is averaged over the evaluation requests.}
\label{tab:space-size}
\fitwidth{%
\begin{tabular}{@{}lcccccccc@{}}
\toprule
\multicolumn{9}{@{}l}{\textbf{TGB}} \\
 & Qwen2.5-7B & Qwen2.5-14B & Llama3.1-8B & Mistral-7B & Qwen3.5-9B & Qwen3-8B & phi-4 & G4-12B \\
\midrule
Keep all & 15 & 15 & 15 & 15 & 15 & 15 & 15 & 15 \\
BM25$^\ast$ & 5.1 & 3.6 & 6.8 & 3.9 & 5.0 & 2.8 & 3.9 & 4.4 \\
Dense$^\ast$ & 5.1 & 3.6 & 6.8 & 3.9 & 5.0 & 2.8 & 3.9 & 4.4 \\
Tool2Vec$^\ast$ & 5.1 & 3.6 & 6.8 & 3.9 & 5.0 & 2.8 & 3.9 & 4.4 \\
LLM as router$^\ddagger$ & 1.6 & 2.7 & 2.5 & 2.6 & 1.8 & 2.2 & 2.0 & 2.1 \\
TTO$^\dagger$ & 12.0 & 15.0 & 3.0 & 1.0 & 2.0 & 4.0 & 12.0 & 15.0 \\
\ours{} & 5.1 & 3.6 & 6.8 & 3.9 & 5.0 & 2.8 & 3.9 & 4.4 \\
\midrule
\multicolumn{9}{@{}l}{\textbf{BFCL}} \\
 & Qwen3-4B & Qwen3.5-4B & Qwen3.5-9B & Gemma4-12B & Phi-4-mini & & & \\
\midrule
Keep all & all & all & all & all & all & & & \\
BM25$^\ast$ & 19.8 & 14.8 & 24.2 & 13.8 & 8.0 & & & \\
LLM as router$^\ddagger$ & 4.7 & 4.7 & 5.0 & 5.5 & 6.6 & & & \\
TTO$^\dagger$ & 4.3 & 5.0 & 9.0 & 3.5 & 0.8 & & & \\
Tool2Vec & 16 & 16 & 16 & 16 & 16 & & & \\
\ours{} & 19.8 & 14.8 & 24.2 & 13.8 & 8.0 & & & \\
\bottomrule
\end{tabular}}
\end{table}

\ours{} keeps 2.8 to 6.8 of the 15 TGB tools (mean 4.4) and documents 8.0 to 24.2 tools per BFCL class; the router keeps 1.6 to 2.7 on TGB and 4.7 to 6.6 on BFCL. TTO ranges from one tool on Mistral-7B to the full registry on Qwen2.5-14B and Gemma4-12B, and Beam Search on Gemma4-12B deploys the empty space for eight of ten families.

\subsection{Constructing Spaces as New Task Types Arrive}

Figure~\ref{fig:stream} uses 120 tools, the 15 authored tools and 105 distractors. Eight families arrive at intervals of 15 steps (which eight depends on the model; for Qwen2.5-7B they are the seven of Figure~\ref{fig:stream} and \texttt{schedule\_gap}); each batch has four requests, a family's first batch uses the full registry, and later batches use its current space with $\varepsilon=0.25$, EMA weight $\alpha=0.3$ and top-$K$ retention with the budget from its 120-tool fit. Table~\ref{tab:stream} extends the main-text example to seven models and three seeds: \ours{} finishes above Keep all in all 21 runs, with serving prompt tokens about one third of Keep all (scoring cost excluded).

\begin{table}[H]
\centering\footnotesize
\caption{Task-arrival streams on seven models and three seeds. Accuracy is credited over arrived families at the end of the stream. Each entry is \ours{} / Keep all. Prompt tokens are averaged over served requests, including exploration and excluding update scoring.}
\label{tab:stream}
\begin{tabular}{lccccc}
\toprule
Model & Families & Seed 0 & Seed 1 & Seed 2 & Prompt tokens \\
\midrule
Qwen2.5-7B & 8 & 0.82 / 0.52 & 0.82 / 0.43 & 0.79 / 0.46 & 641 / 1832 \\
Qwen2.5-14B & 8 & 0.87 / 0.76 & 0.86 / 0.80 & 0.85 / 0.82 & 623 / 1820 \\
Qwen3-8B & 8 & 0.42 / 0.21 & 0.53 / 0.23 & 0.37 / 0.20 & 629 / 1834 \\
Qwen3.5-9B & 8 & 0.62 / 0.30 & 0.59 / 0.32 & 0.61 / 0.29 & 659 / 1860 \\
phi-4 & 8 & 0.77 / 0.48 & 0.77 / 0.53 & 0.74 / 0.49 & 531 / 1537 \\
Mistral-7B & 8 & 0.62 / 0.20 & 0.50 / 0.30 & 0.64 / 0.30 & 706 / 2026 \\
G4-12B & 8 & 0.84 / 0.77 & 0.87 / 0.70 & 0.87 / 0.73 & 637 / 1768 \\
\bottomrule
\end{tabular}
\end{table}

Figure~\ref{fig:whatevolves} shows how request-level evidence forms a ranking and then a space on BFCL and GTA-Atomic. Figure~\ref{fig:tasktool} shows the TGB spaces under the top-$K$ protocol: all contain their chain, with budgets from one tool for the two controls to fourteen for the sensor families.

\begin{figure}[H]
\centering
\includegraphics[width=\linewidth]{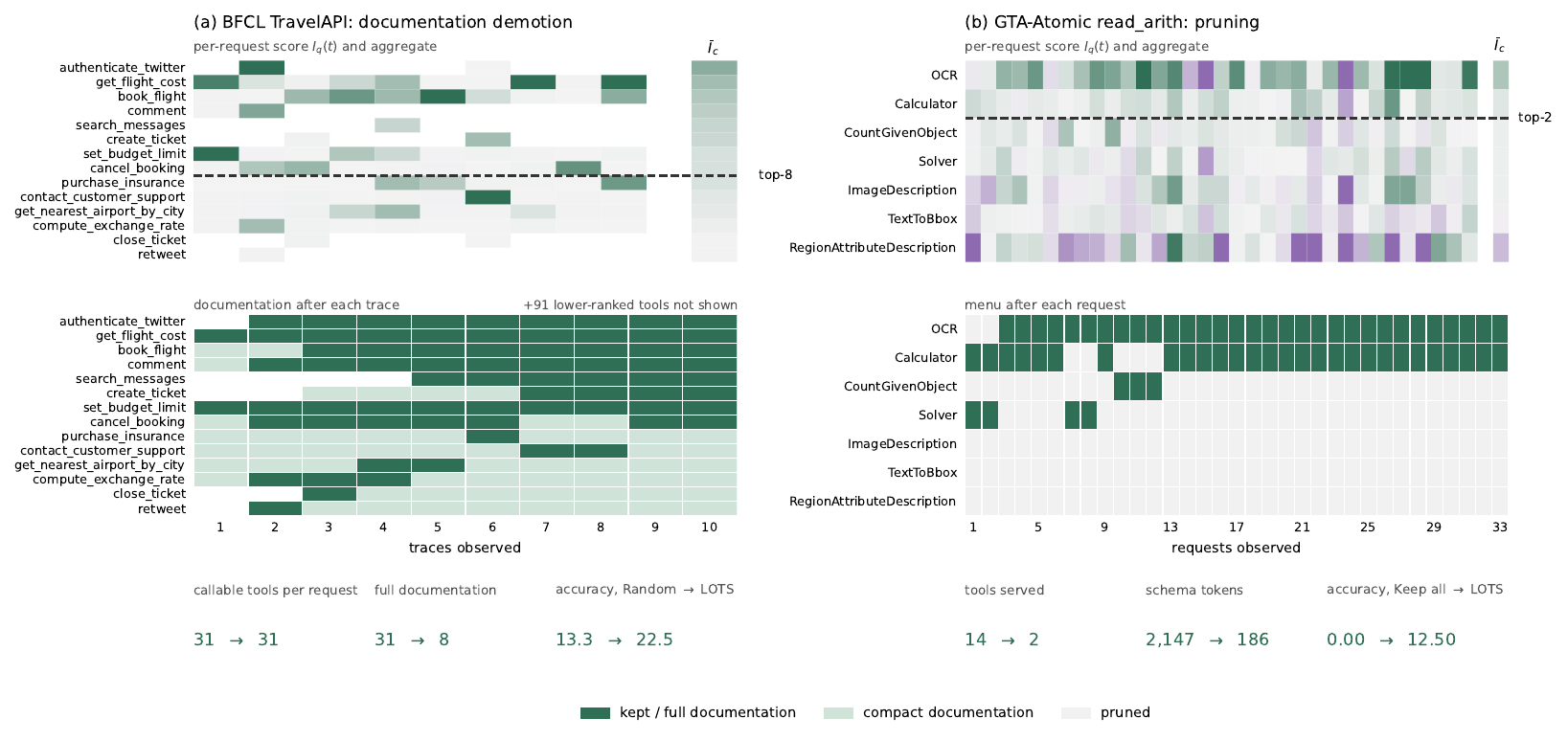}
\caption{From request scores to a task ranking and tool space, Qwen2.5-7B. (a) BFCL \texttt{TravelAPI}, using documentation demotion. (b) GTA-Atomic \texttt{read\_arith}, using pruning, fitted on 33 requests; the accuracies are those of Table~\ref{tab:gta-desc} on the remaining 32.}
\label{fig:whatevolves}
\end{figure}

\begin{figure}[!htb]
\centering
\includegraphics[width=0.9\linewidth]{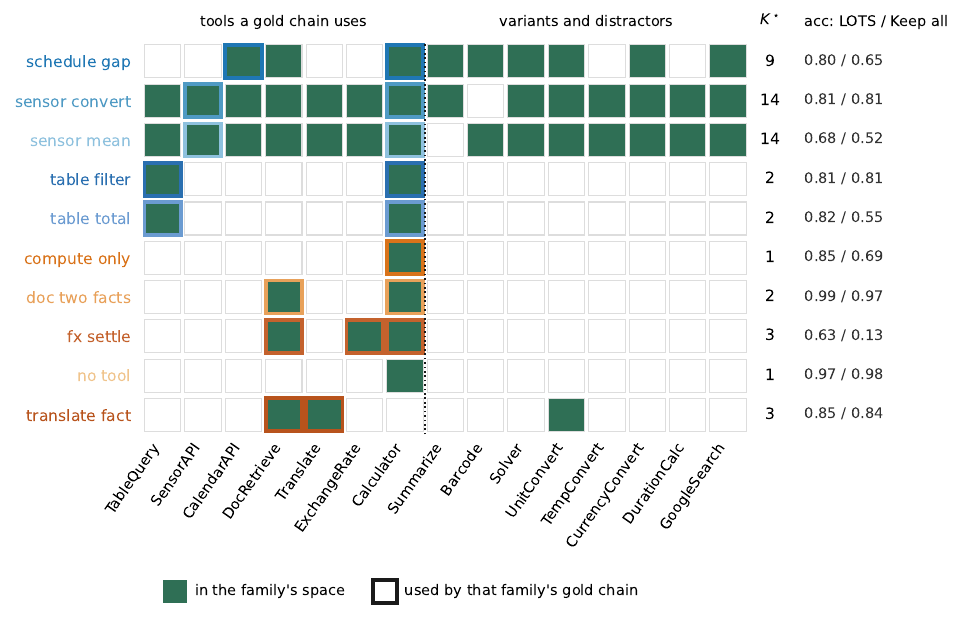}
\caption{Task-specific tool spaces on TGB, Qwen2.5-7B, under the main protocol: each row is fitted on a family's first 80 requests and keeps its top-$K$ tools. Outlined cells mark the family's gold chain; the right margin reports $K$ and accuracy against Keep all on the remaining 320 requests.}
\label{fig:tasktool}
\end{figure}

\subsection{Do Further Updates Improve an Existing Space?}
\label{app:within}
\label{app:trajectories}

We fix the task and request distribution and compare an updated space with one frozen after its first batch. In Figure~\ref{fig:taskloop-all}, each TGB family runs eight batches of 40 requests with Qwen2.5-7B and the 15-tool registry; three arms share the request order and evaluation set: Keep all, a space frozen after the first update, and the \ours{} update applied after every batch (support-frequency retention, EMA $\alpha=0.2$, threshold 5\%). The first update carries most of the gain: the updated spaces reach 0.752 after the first batch and 0.757 after the eighth, against 0.750 for the frozen arm and 0.705 for Keep all, and prompt lengths stay at 56--205 tokens per request against 200--345 under the full menu. Two of the ten spaces change after the first update; the others confirm their configuration.

\begin{figure}[!htb]
\centering
\includegraphics[width=\linewidth]{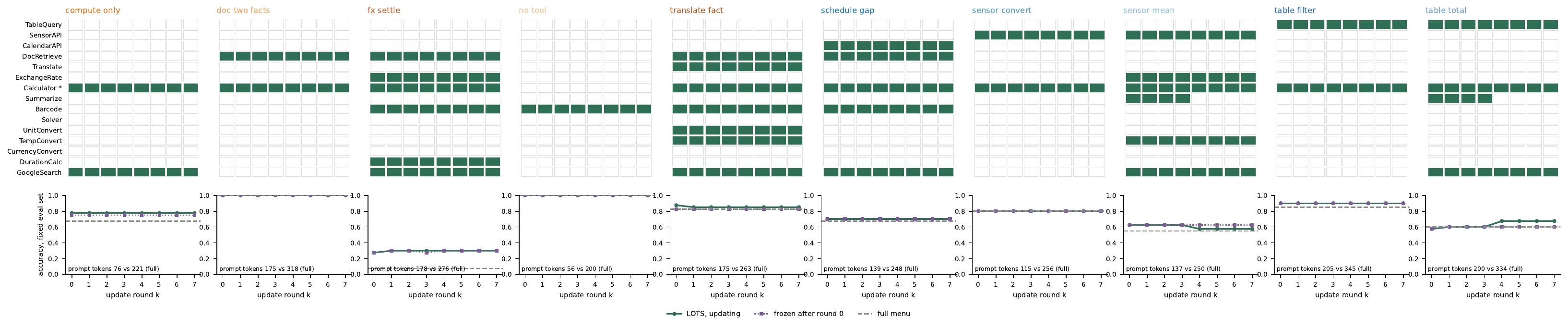}
\caption{Stationary-task update trajectories, Qwen2.5-7B: 15-tool registry, eight batches of 40 requests. Top: tools retained by the updating arm after each round; outlines mark entries. Bottom: accuracy on each family's fixed evaluation set for the updating, frozen and Keep-all configurations.}
\label{fig:taskloop-all}
\end{figure}

\textbf{Update rules under a serving cap.}
When the registry cannot be served in full, the first batch sees only a random sample of it and later requests are explored with probability $\varepsilon$ (Appendix~\ref{app:aggregation}). Table~\ref{tab:explore-ctrl} compares three update rules in this setting under the same request order, exploration probability and seeds: the \ours{} update, a size-matched random draw from the tools with evidence so far, and the union of all such tools. Each rule is a separate run, so the evidence differs once the spaces differ. \ours{} reaches 0.27 with four tools that hold 41\% of the required tools, against 0.01 for the random space and 0.02 for the union; the union's Gold of 0.29 follows from the cap, since its space already holds about 53 tools after the first round, so exploration adds nothing further, whereas the four-tool \ours{} space keeps discovering tools.

\begin{table}[H]
\centering\footnotesize
\caption{Update rules under a serving cap, Qwen2.5-7B: eight families, 120-tool registry, 30 rounds of four requests, $\varepsilon=0.5$, three seeds; each rule is a separate run with the same request order, exploration probability and seeds. An exploring request receives the current space plus at most $16-|S|$ further tools, and a space larger than 16 tools is served in full. Accuracy is read on each family's fixed evaluation set under the current space, averaged over the rounds and at the final round; size is the final space; gold is the share of the family's required tools in the final space.}
\label{tab:explore-ctrl}
\begin{tabular}{lccccc}
\toprule
Update rule & Runs & Mean acc. & Final acc. & Size & Gold \\
\midrule
\ours{} update (top-$K$) & 24 & 0.162 & 0.265 & 4.0 & 0.41 \\
Random, same $K$ & 24 & 0.010 & 0.006 & 4.0 & 0.00 \\
Union of exposed tools & 24 & 0.023 & 0.023 & 52.7 & 0.29 \\
\bottomrule
\end{tabular}
\end{table}

\FloatBarrier
\input{appendix_worked_examples}

\section{Analysis of Tool Scoring}
\label{app:diagnostics}

The following diagnostics examine the evidence used by the score, the attribution rule, and the stability of the aggregate ranking. They use different scopes to isolate these choices: request-level selection, pooled configurations across TGB families, and the task-specific spaces of the main experiments. Request-level thresholding retains tools above a tuned threshold $\tau^\star$; support-frequency retention uses the rule defined in Appendix~\ref{app:within}, with an explicitly reported exception for Llama-3.1-8B. Results from different scopes are not treated as matched comparisons.

\subsection{Top-$K$ Selection as Surrogate Maximization}
\label{app:theory}

\ours{} measures each tool's contribution with the other tools present,
then selects a task-specific space from the resulting scores.
This section restates Theorem~1 of Section~\ref{sec:lots-update} with its
setting and proof: the selection exactly optimizes an additive
approximation to the fixed-output likelihood objective.
Bounded tool interactions then give an explicit bound on the
approximation error and the resulting selection gap.
Consider one task and a fixed collection of requests $Q$ served
under a common tool menu $S$, with $m=|S|$.
For each request $x$, keep its recorded output $y_x$ fixed.
Let $h_x(A)$ denote the scoring context obtained by retaining
the evidence associated with $A\subseteq S$, using a fixed
intervention rule and leaving tool documentation unchanged.
Define
\begin{equation}
F(A)
=
\frac{1}{|Q|}
\sum_{x\in Q}
\mathcal{L}_\theta(y_x\mid h_x(A)).
\label{eq:theory-objective}
\end{equation}
Thus, $F$ measures support for the recorded outputs, rather than
the accuracy of newly generated answers.
The task-averaged \ours{} score is
\begin{equation}
I(s)=F(S)-F(S\setminus\{s\}).
\label{eq:theory-loo}
\end{equation}

These scores define the additive approximation
\begin{equation}
\widehat F(A)
=
F(S)-\sum_{s\in S\setminus A}I(s).
\label{eq:theory-surrogate}
\end{equation}
It agrees exactly with $F$ at the full menu and at every
single-tool deletion.
Since
$\widehat F(A)=F(S)-\sum_{s\in S}I(s)+\sum_{s\in A}I(s)$,
retaining the $K$ highest-scoring tools exactly maximizes
$\widehat F(A)$ subject to $|A|=K$.
The approximation arises only when several tools are removed
together: their joint effect need not equal the sum of their
individual leave-one-out effects.

\textbf{Tool interactions.}
Following the discrete-difference formulation of set-function
interactions \citep{sundararajan2020shapley}, define
\begin{equation}
\Delta_{s,t}F(B)
=
F(B\cup\{s,t\})
-F(B\cup\{s\})
-F(B\cup\{t\})
+F(B),
\label{eq:theory-interaction}
\end{equation}
for distinct $s,t\notin B$.
This quantity measures how the presence of one tool changes
the marginal contribution of the other.
Positive values indicate complementarity in the likelihood
objective; negative values indicate substitutability.
The following result bounds the error of the \ours{} approximation
using these interactions.

\textbf{Theorem 1} (selection under bounded interactions)\textbf{.}
Fix a budget $K$ and let $r=m-K$.
Suppose that, for some $\beta\geq 0$,
\begin{equation}
|\Delta_{s,t}F(B)|\leq\beta
\label{eq:theory-assumption}
\end{equation}
for every $B\subseteq S$ with $K\leq |B|\leq m-2$
and every distinct $s,t\in S\setminus B$.
Then every set $A$ of size $K$ satisfies
\begin{equation}
|F(A)-\widehat F(A)|
\leq
\binom{r}{2}\beta.
\label{eq:theory-error}
\end{equation}
Let $A_{\mathrm{LOTS}}$ contain the $K$ highest-scoring tools,
and let $A^\star\in\arg\max_{|A|=K}F(A)$.
Then
\begin{equation}
0
\leq
F(A^\star)-F(A_{\mathrm{LOTS}})
\leq
r(r-1)\beta.
\label{eq:theory-gap}
\end{equation}

\textbf{Proof.}
For any $A$ with $|A|=K$, order the removed tools as
$S\setminus A=\{d_1,\ldots,d_r\}$.
Set $S_0=S$ and $S_i=S_{i-1}\setminus\{d_i\}$.
The likelihood loss at deletion $i$ is
\[
\ell_i=F(S_{i-1})-F(S_i).
\]
Its difference from the full-context score $I(d_i)$
comes from the preceding $i-1$ deletions.
Each preceding deletion changes this marginal by a
second-order difference of the form
\eqref{eq:theory-interaction}.
All corresponding contexts have sizes between $K$ and $m-2$,
so the assumption gives
\[
|\ell_i-I(d_i)|\leq(i-1)\beta.
\]
Telescoping the deletion losses yields
\[
F(A)=F(S)-\sum_{i=1}^{r}\ell_i.
\]
Consequently,
\[
|F(A)-\widehat F(A)|
\leq
\sum_{i=1}^{r}|\ell_i-I(d_i)|
\leq
\sum_{i=1}^{r}(i-1)\beta
=
\binom{r}{2}\beta.
\]
Write $\delta=\binom{r}{2}\beta$.
Because $A_{\mathrm{LOTS}}$ maximizes $\widehat F$
among sets of size $K$,
\[
F(A^\star)
\leq
\widehat F(A^\star)+\delta
\leq
\widehat F(A_{\mathrm{LOTS}})+\delta
\leq
F(A_{\mathrm{LOTS}})+2\delta,
\]
which proves the selection bound.
For $r=0$ or $r=1$, the approximation is exact.
\hfill$\square$

The bound requires that a tool's marginal contribution not change sharply as other tools are removed; each interaction in \eqref{eq:theory-interaction} costs four fixed-output evaluations, so sampled pairs can test it. The analysis covers one common-menu batch and the fixed-output objective, not running scores across menus or execution accuracy.

\subsection{Scoring Targets and Attribution}
\label{sec:exp-why}
\label{app:tgb}
\label{app:span-ablation}

Table~\ref{tab:tgb-ablation} tests the scoring target on Qwen2.5-7B with request-level selection on TGB: scoring the same own output under every intervention reaches 0.724, against 0.648 when each candidate context generates its own answer and 0.796 with gold answers, which use labels; removing the tuned threshold grows the retained set from 3.7 to 10.4 tools and lowers accuracy to 0.698.

\begin{table}[H]
\centering\footnotesize
\caption{Request-level scoring variants on TGB, Qwen2.5-7B. Gold-answer scoring uses labels; per-candidate generation scores a different string under each intervention. These configurations differ in retained size.}
\label{tab:tgb-ablation}
\begin{tabular}{lcc}
\toprule
Variant & Tools kept & Accuracy \\
\midrule
Fixed own output, LOO, $\tau^\star$ & 3.7 & 0.724 \\
Fixed own output, LOO, $\tau=0$ & 10.4 & 0.698 \\
Per-candidate generation ($L_{\mathrm{self}}$) & --- & 0.648 \\
Gold-answer scoring & 12.2 & 0.796 \\
\bottomrule
\end{tabular}
\end{table}

Figure~\ref{fig:score-example} illustrates the attribution on one request. On GTA-Atomic, scoring the numeric answer expressions selects Calculator+OCR in all five scoring runs (16.92\%), whereas scoring the whole output selects OCR+RegionAttributeDescription (7.69\%); 9 to 10 of the 65 outputs per run contain no numeric expression and are excluded. TGB scores its short answer and BFCL the final assistant message.

\begin{figure}[H]
\centering
\includegraphics[width=0.72\linewidth]{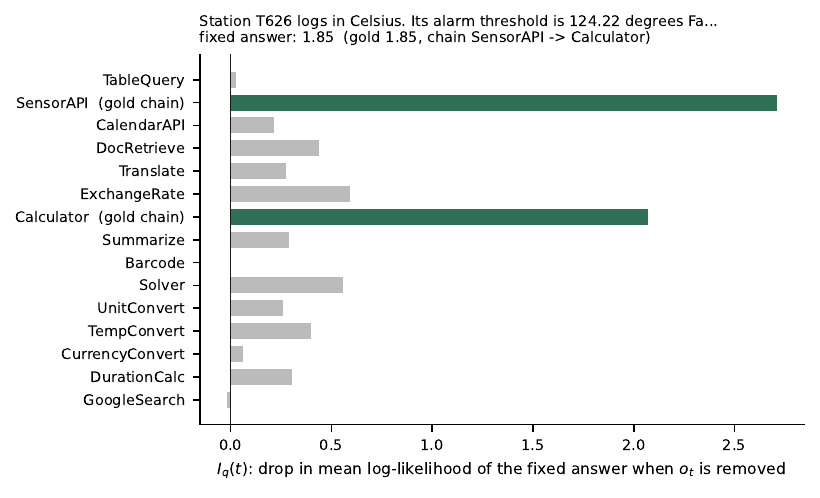}
\caption{Tool scores for one TGB \texttt{sensor\_convert} request, Qwen2.5-7B. The full-menu answer is fixed while each tool is removed using the TGB intervention, including downstream propagation where applicable. The two gold-chain tools carry the score.}
\label{fig:score-example}
\end{figure}

\textbf{Attribution in compositional tasks.}
\label{app:interactions}
In Table~\ref{tab:tgb-shape}, a document-only selector reaches 0.418 on Qwen2.5-7B and per-request BM25 at four tools 0.193 (0.182 on Qwen3.5-9B); this request-level variant is separate from the task-level, budget-matched BM25 of Table~\ref{tab:tgb-family}. Exact Shapley values over the seven chain tools correlate between likelihood and correctness valuations (Pearson 0.95, Spearman 0.96, one sign disagreement), yet retaining the positive-valued tools reaches 0.370 and 0.360 against 0.795 for the seven-tool union: averaging over coalitions penalizes a tool whose partners are often absent.

\begin{table}[H]
\centering\footnotesize
\caption{TGB coalition diagnostics over the seven chain tools. Positive Shapley values define the retained sets; the best subset is selected in hindsight. The final two rows are a separate single-tool diagnostic. Request-level \ours{} uses a different scope and is reported in Table~\ref{tab:tgb-ablation}.}
\label{tab:tgb-shape}
\fitwidth{%
\begin{tabular}{lcc}
\toprule
Diagnostic & Value & Accuracy \\
\midrule
Shapley$^{\mathrm{acc}}$ vs Shapley$^{\mathrm{lik}}$, Pearson / Spearman & $+0.95$ / $+0.96$ & \\
Sign disagreements & 1 of 7 tools & \\
Keep $\{\phi_i^{\mathrm{lik}}>0\}$ & 4 tools & 0.370 \\
Keep $\{\phi_i^{\mathrm{acc}}>0\}$ & 5 tools & 0.360 \\
Union of all chain tools & 7 tools & 0.795 \\
Best subset in hindsight & & 0.870 \\
\midrule
Single-tool score, AUROC vs one-tool success & 0.789 & \\
Best single tool deployed alone & & 0.135 \\
\bottomrule
\end{tabular}}
\end{table}

\textbf{Enumerating GTA-Atomic tool combinations.}
Table~\ref{tab:gta-combo} evaluates all 15 pairs of six local tools on the 65 \texttt{read\_arith} requests with Qwen2.5-7B. \ours{} improves over Keep all by 7.69 points (95\% interval $[-1.54,+16.92]$, $p=0.18$) and over No tools by 10.77 ($[+3.08,+20.00]$, $p=0.039$). In an independently built stack, Calculator+OCR reaches 21.54\% against 7.69\% for Keep all, 9.23\% for size-matched Random and 23.08\% for the gold oracle, and pooling the router's per-request choices over the task also yields Calculator+OCR.

\begin{table}[H]
\centering\footnotesize
\caption{GTA-Atomic combination diagnostic, Qwen2.5-7B, using the same 65 \texttt{read\_arith} requests for scoring and evaluation. Accuracy averages replicate executions. $\Delta$ is the paired difference of \ours{} answer-expression scoring against the row, with a 95\% bootstrap interval in percentage points.}
\label{tab:gta-combo}
\fitwidth{%
\begin{tabular}{@{}llccc@{}}
\toprule
Space & Tools & Replicates & Accuracy (\%) & $\Delta$ (pp) \\
\midrule
No tools & none & 5 & 6.15 & $+10.77$ $[+3.08,+20.00]$ \\
Keep all (local tools) & all six & 5 & 9.23 & $+7.69$ $[-1.54,+16.92]$ \\
No-API oracle & per-request gold & 5 & 15.38 & $+1.54$ $[0.00,+4.62]$ \\
\midrule
Calculator only & Calculator & 3 & 4.62 & \\
OCR only & OCR & 3 & 15.38 & \\
\ours{}, answer expressions & Calculator $+$ OCR & 5 & \textbf{16.92} & \\
\ours{}, whole output & OCR $+$ RegionAttributeDescription & 3 & 7.69 & \\
\ours{} top-4, answer expressions & four tools & 3 & 15.38 & \\
\midrule
Best of 15 pairs (hindsight) & OCR $+$ Solver & 3 & 23.08 & \\
Rank of the \ours{} pair & \multicolumn{4}{l}{2 of 15} \\
\bottomrule
\end{tabular}}
\end{table}

\textbf{Likelihood scoring versus an LLM trace judge.}
\label{app:trace-judge}
The judge replaces the likelihood score: the serving model answers each fitting request under the full menu, is asked which tools' outputs support its frozen answer~\citep{zheng2023judging}, tools are ranked by selection frequency, and the top-$K$ form the space. In Table~\ref{tab:trace-judge}, the judge improves over Keep all on four of five models, and likelihood scoring is higher on all five, by 0.012 on Qwen2.5-14B to 0.206 on Qwen2.5-7B.

\begin{table}[H]
\centering\footnotesize
\caption{Trace-judge ablation on TGB: the space is built from LLM judgments of the frozen traces instead of likelihood scores, with the same task-level aggregation and budget $K$.}
\label{tab:trace-judge}
\begin{tabular}{lccccc}
\toprule
 & Qwen2.5-7B & Qwen2.5-14B & Mistral-7B & Qwen3.5-9B & Qwen3-8B \\
\midrule
Keep all & 0.696 & 0.880 & 0.464 & 0.509 & 0.565 \\
Trace judge & 0.623 & 0.906 & 0.558 & 0.613 & 0.689 \\
\ours{} & \textbf{0.829} & \textbf{0.918} & \textbf{0.658} & \textbf{0.652} & \textbf{0.815} \\
\bottomrule
\end{tabular}
\end{table}

\subsection{Deletion Operation in the Leave-One-Out Score}
\label{app:deletion}

On TGB, removing $s$ re-executes the chain, so a downstream tool that needs $s$'s output returns an error. We compare this \emph{re-execute} deletion with \emph{evidence only}, which deletes just the recorded output line of $s$, on the same 80 fitting requests per family, the same traces and the same frozen answers (Qwen2.5-7B).

\begin{table}[H]
\centering\small\setlength{\tabcolsep}{4pt}
\caption{Two deletion operations on Qwen2.5-7B: chain identification on the fitting requests and evaluation accuracy of the spaces they select at fixed budgets. The mean rank is over the chain tools of the nine tool-using families (1 is best).}
\label{tab:tgb-deletion}
\fitwidth{%
\begin{tabular}{@{}lccccccc@{}}
\toprule
 & \multicolumn{3}{c}{Chain identification} & \multicolumn{3}{c}{Evaluation accuracy (\%)} \\
\cmidrule(lr){2-4}\cmidrule(lr){5-7}
Deletion & Mean rank, upstream & Mean rank, downstream & Chain in top-$|G|$ & $K=2$ & $K=4$ & $K=8$ \\
\midrule
Re-execute & 1.11 & 2.00 & 9 / 9 & 66.8 & 71.4 & 72.3 \\
Evidence only & 1.56 & 1.56 & 9 / 9 & 66.6 & 71.4 & 72.3 \\
\bottomrule
\end{tabular}}
\end{table}

Both deletions place the complete chain of every tool-using family within its top $|G|$ positions, their pooled scores correlate at 0.91, the top-4 sets agree in all ten families and the top-2 sets in nine, and the deployed accuracy is identical at $K=4$ and $K=8$ and differs by 0.2 points at $K=2$. Evidence-only deletion halves the upstream scores (DocRetrieve in \texttt{fx\_settle} 1.21 to 0.87, SensorAPI in \texttt{sensor\_mean} 1.86 to 0.55), because the recorded downstream output still carries the upstream information; the mean upstream rank moves from 1.11 to 1.56 and Calculator moves above ExchangeRate in \texttt{fx\_settle}. Re-execution credits an upstream tool with its effect through its dependents, the quantity a deployment removes when it drops the tool.

\subsection{From Likelihood Scores to Execution Utility}
\label{app:likelihood-utility}

For each of eight models and ten TGB families, the scores cached from the first 80 requests select the $K$ highest-scoring tools (Top-$K$) and one random subset of the same size, evaluated on the remaining 320 requests. Excluding \texttt{no\_tool}, seven models separate the reference-chain tools from the rest with an AUROC of 1.000 on all nine families and Gemma4-12B with 0.992, so 71 of 72 model--family pairs are separated perfectly. In Table~\ref{tab:likelihood-utility}, Top-$K$ beats Random and Keep all on every model, averaging 0.819 against 0.222 and 0.694: the ranking, not the size of the space, carries the gain. Appendix~\ref{app:fixedk} deploys the same rankings at fixed budgets.

\begin{table}[t]
\centering
\small
\setlength{\tabcolsep}{5pt}
\caption{
\textbf{Likelihood ranking and execution utility on TGB.}
Accuracy is averaged across ten task families, using 320
evaluation requests per family.
The three configurations are those of Table~\ref{tab:tgb-family}:
Top-$K$ and Random use identical family-specific budgets,
and Random uses a single fixed-seed draw.
Bold denotes the best configuration for each model.
}
\label{tab:likelihood-utility}
\begin{tabular}{lccc}
\hline
Model & Top-$K$ & Random & Keep all \\
\hline
Qwen2.5-7B  & \textbf{0.829} & 0.234 & 0.696 \\
Qwen2.5-14B & \textbf{0.918} & 0.193 & 0.880 \\
Llama3.1-8B & \textbf{0.848} & 0.413 & 0.777 \\
Mistral-7B  & \textbf{0.658} & 0.115 & 0.464 \\
Qwen3.5-9B  & \textbf{0.652} & 0.275 & 0.509 \\
Qwen3-8B    & \textbf{0.815} & 0.101 & 0.565 \\
phi-4       & \textbf{0.921} & 0.243 & 0.797 \\
Gemma4-12B  & \textbf{0.914} & 0.200 & 0.867 \\
\hline
Mean        & \textbf{0.819} & 0.222 & 0.694 \\
\hline
\end{tabular}
\end{table}

\subsection{Scoring Models, Aggregation, and Incorrect Outputs}
\label{app:scorer}

Rescoring the same instruct-model traces and frozen outputs with the base model, then serving with the instruct model, selects the same tools for Qwen2.5-14B, Mistral-7B and Qwen3-4B and gives higher accuracy for the other three models (Figure~\ref{fig:tgb-scorer}). The instruct scorer is closer to the gold chain on five of six models, and base scores are smaller on five, so fewer tools pass the fixed threshold of the support-frequency rule used here; at a fixed $K$, top-$K$ selection is invariant to a uniform rescaling of the scores.

\begin{figure}[H]
\centering
\includegraphics[width=0.55\linewidth]{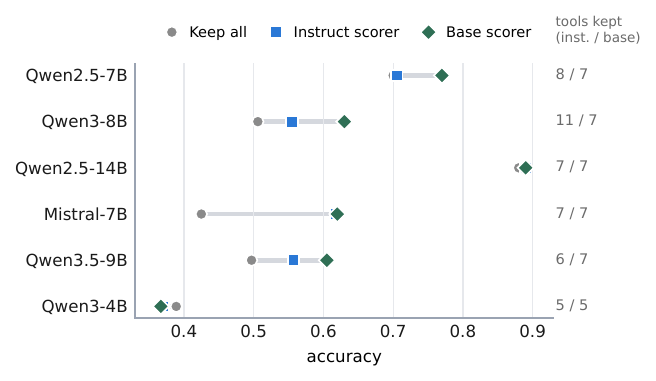}
\caption{Scoring-model diagnostic on TGB: accuracy under Keep all and under the spaces fitted with an instruct-model or a base-model scorer, from frozen instruct-model outputs and a fixed support-frequency threshold. The instruct model serves every configuration; the right margin gives the number of tools each scorer keeps.}
\label{fig:tgb-scorer}
\end{figure}

\textbf{Incorrect outputs and aggregation.}
\label{app:wrong-answers}
Figure~\ref{fig:tgb-correct} uses pooled configurations fitted from 3,000 traces per model and evaluated on 400 separate requests, with correct-only and wrong-only subsets and synthetic corruption of a random fraction of the scored answers ($1.17g+3$ for numeric gold $g$, three draws). On Qwen2.5-7B the correct-only and wrong-only rankings correlate at Spearman 0.911; support-frequency retention keeps its eight-tool space through 50\% corruption, whereas mean-score retention falls to 0.620 at 50\% and 0.102 at 75\%; Llama-3.1-8B declines from 0.765 to 0.690 at 50\% corruption. Incorrect outputs retain useful evidence, with robustness depending on the model and the retention rule; these pooled diagnostics do not cover the per-task top-$K$ protocol.

\begin{figure}[!htb]
\centering
\includegraphics[width=\linewidth]{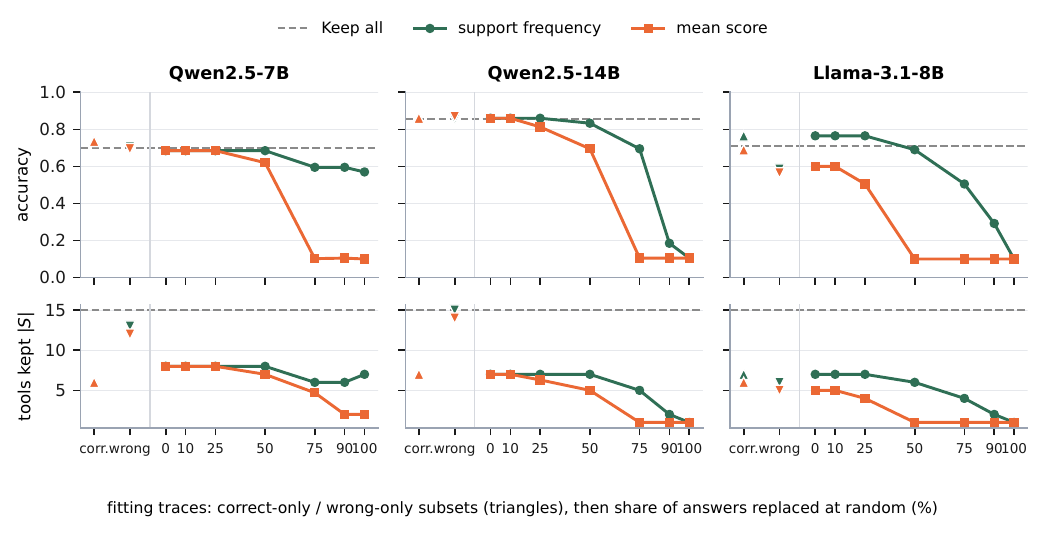}
\caption{Incorrect-output diagnostic: 3,000 fitting traces per model and 400 separate evaluation requests. Triangles give label-filtered subsets (correct only, wrong only); lines replace a growing share of answers at random (three draws). Top: accuracy; bottom: tools kept. Support-frequency retention uses a 3\% threshold for Llama-3.1-8B and 5\% otherwise; mean-score retention uses $\tau=0.1$.}
\label{fig:tgb-correct}
\end{figure}

\subsection{Sensitivity to the Number of Traces}
\label{app:traces}

Figure~\ref{fig:tgb-traces} subsamples the 3,000-trace Qwen2.5-7B pool thirty times at each size: with 100 traces the ranking reaches a Spearman correlation of 0.957 and a Jaccard overlap of 0.95 with the full-pool fit, with 400 traces 0.986 and 0.97. The main experiments fit from 80 requests per family.

\begin{figure}[H]
\centering
\includegraphics[width=0.8\linewidth]{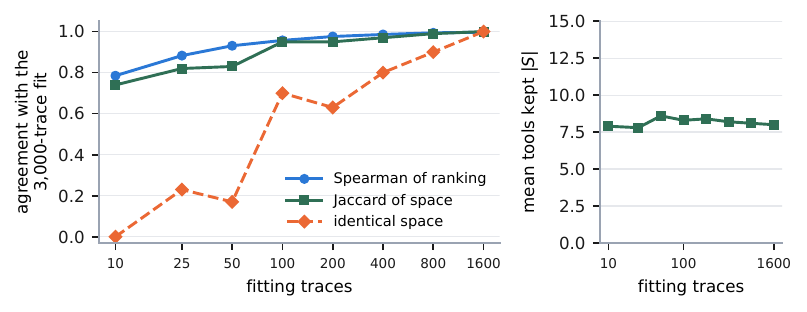}
\caption{Trace-count sensitivity, TGB and Qwen2.5-7B. Thirty subsamples per size are compared with the configuration fitted from all 3,000 traces: rank correlation, Jaccard overlap of the space, and the fraction of subsamples with an identical space (left); mean space size (right).}
\label{fig:tgb-traces}
\end{figure}

\FloatBarrier

\section{Computational Cost and Registry Scaling}
\label{app:efficiency}

A trace with $m$ interventions needs at most $m+1$ teacher-forced sequences, one with the full context and one per intervention; no replacement answer is generated. Table~\ref{tab:cost} reports selection and serving input tokens: the GTA-Atomic fit scores 33 requests over seven candidates (264 sequences) at fixed $K=2$, and the BFCL rows use the 101-record tune protocol at fixed $K=16$. A router pays its selection call on every request; a frozen space pays once. The full-menu runs that produced the fitting traces are excluded from the selection column; charging them adds about 98k prompt tokens on GTA-Atomic and 29M on BFCL. Serving tokens count every agent turn. On GTA-Atomic the 14-tool schema has 2,147 tokens and the OCR+Calculator schema 186, and Qwen2.5-7B processes 2,956 against 1,572 prompt tokens per request; the size-matched random menu uses 671 tokens but also makes fewer calls. On TGB, serving under the fitted spaces uses 143 to 181 prompt tokens per request against 251 to 319 for the full fifteen-tool menu, 33 to 46\% fewer.

\begin{table}[!htb]
\centering\scriptsize\setlength{\tabcolsep}{2.2pt}
\caption{Input-token costs of selecting and serving tool spaces. \ours{} selection counts likelihood scoring. Router selection is paid per request; other selection costs are per fit. A lower serving cost comes with lower accuracy for Router and Random: both are below \ours{} on every model of Tables~\ref{tab:tgb-family}--\ref{tab:gta-combined}. Inference counts all serving prompts for one request or BFCL record. Pays back divides reported \ours{} selection tokens by serving tokens saved against Keep all. -- denotes an inapplicable or unrun configuration. The Gemma4-12B TTO and Beam Search cells count the local search, 13 and 27 candidate spaces evaluated on the 33 fitting requests.}
\label{tab:cost}
\fitwidth{%
\begin{tabular}{@{}l@{\hspace{5pt}}cccc@{\hspace{7pt}}cccccc@{\hspace{7pt}}c@{}}
\toprule
 & \multicolumn{4}{c}{Selection} & \multicolumn{6}{c}{Inference} & \ours{} \\
\cmidrule(lr){2-5}\cmidrule(lr){6-11}
Model & Router & TTO$^\dagger$ & Beam$^\dagger$ & \ours{} & Keep all & Random & Router & TTO$^\dagger$ & Beam$^\dagger$ & \ours{} & Pays back \\
\midrule
\multicolumn{12}{@{}l}{\textbf{BFCL}: fit on 101 tune records; inference sums all turns; payback is in records} \\
Q3-4B & 1,513 & 1.72B & -- & 60.41M & 289k & 127k & 31k & 225k & -- & 174k & 526 \\
Q3.5-4B & 1,530 & 1.66B & -- & 59.88M & 243k & 146k & 47k & 202k & -- & 190k & 1,130 \\
Q3.5-9B & 1,530 & 1.24B & -- & 61.49M & 352k & 161k & 46k & 219k & -- & 238k & 539 \\
Q3-8B & 1,513 & 1.37B & -- & 55.61M & 207k & 121k & 43k & 164k & -- & 162k & 1,244 \\
G4-12B & 1,569 & 0.95B & -- & 63.39M & 198k & 95k & 33k & 125k & -- & 130k & 931 \\
Phi-4-mini & 1,490 & 0.61B & -- & 69.27M & 147k & 67k & 27k & 89k & -- & 88k & 1,179 \\
\addlinespace[3pt]
\multicolumn{12}{@{}l}{\textbf{GTA-Atomic} \texttt{read\_arith}, 33/32 diagnostic split: fit on 33 requests; inference and payback units are requests} \\
Q2.5-3B & 385 & 0.71M & 1.95M & 101k & 9,697 & 3,804 & 3,018 & 3,525 & 3,525 & 4,935 & 22 \\
Q2.5-7B & 385 & 0.27M & 0.63M & 101k & 2,956 & 671 & 550 & 988 & 1,095 & 1,572 & 77 \\
Q3-8B & 385 & 0.42M & 1.29M & 101k & 9,727 & 3,968 & 1,744 & 3,555 & 2,115 & 1,594 & 12 \\
Q2.5-14B & 385 & 0.24M & 0.59M & 101k & 3,382 & 542 & 679 & 716 & 634 & 867 & 40 \\
Q3.5-9B & 401 & 0.64M & 1.71M & 102k & 12,617 & 3,726 & 2,296 & 5,550 & 5,550 & 4,139 & 12 \\
phi-4 & 381 & 0.88M & 2.38M & 86k & 15,174 & 5,835 & 884 & 6,169 & 6,169 & 5,774 & 9 \\
G4-12B & 385 & 2.40M & 5.25M & 102k & 6,531 & 808 & 1,304 & 385 & 1,392 & 1,275 & 19 \\
\bottomrule
\end{tabular}}
\end{table}

For a frozen configuration with fitting cost $C_{\mathrm{fit}}$ and per-request serving costs $c_{\mathrm{full}}>c_{\mathrm{space}}$, the break-even count is
\begin{equation}
N_{\mathrm{break}}=\frac{C_{\mathrm{fit}}}{c_{\mathrm{full}}-c_{\mathrm{space}}}.
\label{eq:app-breakeven}
\end{equation}
Figure~\ref{fig:amortize} plots this account for GTA-Atomic; token break-even differs from latency break-even, since teacher-forced scoring, cached prefill and decoding cost different amounts per token. A deployment that keeps updating also pays for later scoring rounds, so over $R$ rounds the total is
\begin{equation}
C_{\mathrm{total}}=\sum_{k=0}^{R-1}\left(C_{\mathrm{serve},k}+C_{\mathrm{score},k}+C_{\mathrm{extra\ trace},k}\right),
\end{equation}
and Table~\ref{tab:stream} reports the serving component only.

\begin{figure}[H]
\centering
\includegraphics[width=\linewidth]{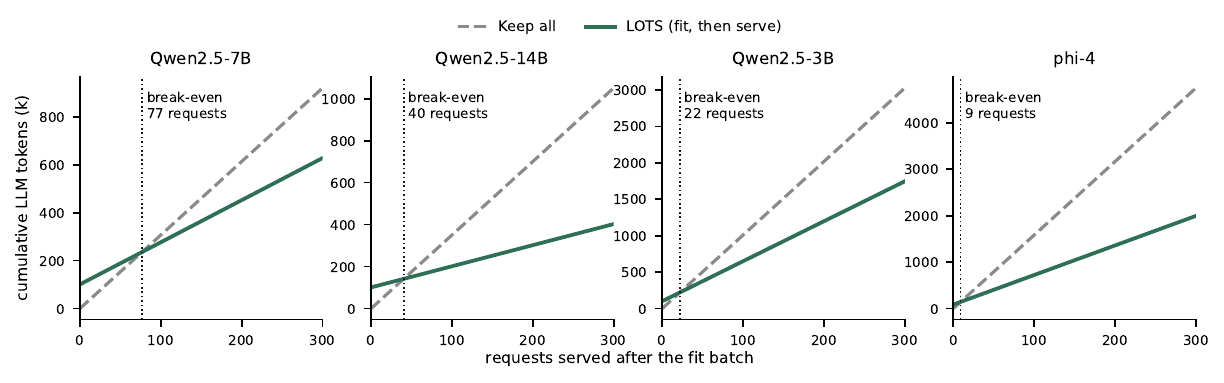}
\caption{GTA-Atomic \texttt{read\_arith} token amortization. \ours{} starts with likelihood-scoring input tokens for the 33 fitting requests and accumulates serving input tokens; Keep all starts at zero and accumulates its own serving cost. The plotted account excludes any additional cost of collecting fitting traces.}
\label{fig:amortize}
\end{figure}

\subsection{Scaling with Registry Size}
\label{app:registry}

Figure~\ref{fig:tgb-scale} appends synthetic distractors to the 15-tool registry; two rounds each serve 400 requests from all ten families, the second under one pooled configuration fitted from the first. Full-menu accuracy falls from 0.677 at 15 tools to 0.508 at 120 on Qwen2.5-7B, where the fitted 43-tool configuration reaches 0.598; Qwen3.5-9B's fitted configurations keep about seven tools and average 0.607, 0.612, 0.594, and 0.538 over four seeds at 30, 60, 120, and 240 tools, while its full menu falls from 0.446 to 0.345; Figure~\ref{fig:teaser}(c) plots these means. Scoring time grows from 176 seconds at 15 tools to 8,158 seconds at 120 on one L40S, and to 3,437 seconds under the fitted configuration.

\begin{figure}[!htb]
\centering
\includegraphics[width=\linewidth]{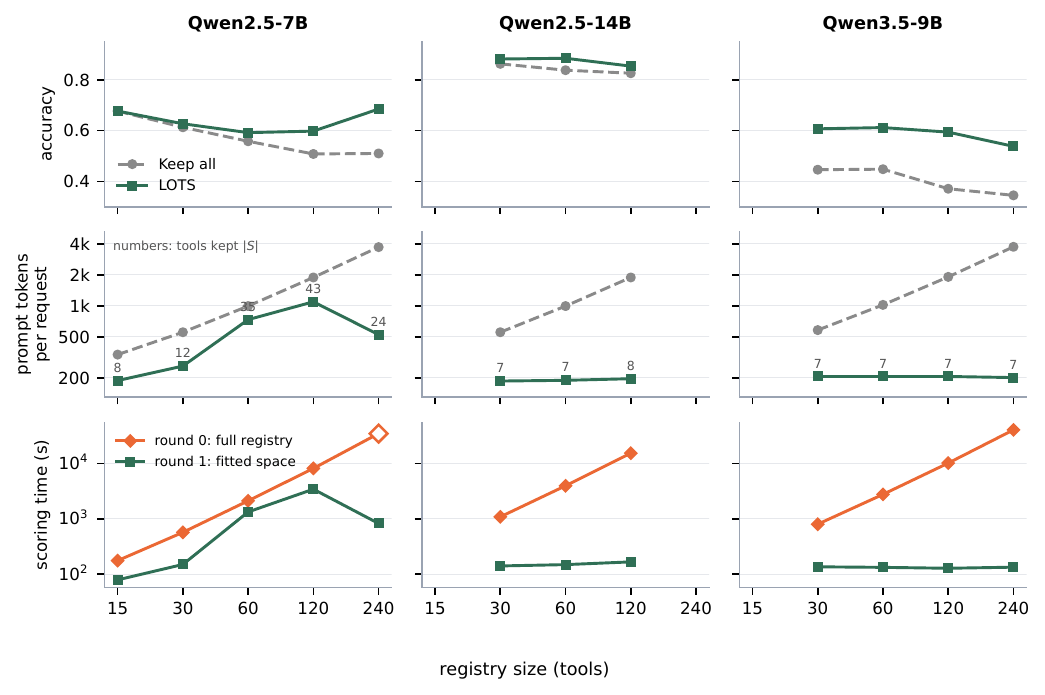}
\caption{TGB registry-size diagnostic with one pooled configuration and two rounds of 400 requests. Top: accuracy under Keep all and under the fitted space. Middle: serving prompt tokens per request, with the number of retained tools above each point. Bottom: scoring time on one L40S for round 0 (full registry) and round 1 (fitted space); this is scoring time, not generation latency. The hollow marker uses a different batch cap. Qwen2.5-7B and Qwen2.5-14B average three seeds at 30 to 120 tools and Qwen3.5-9B averages four seeds for accuracy and tokens; the other points, and all timings, use one seed.}
\label{fig:tgb-scale}
\end{figure}

\end{document}

%% file: math_commands.tex
\usepackage{amsmath,amsfonts,bm}

\def\eqref#1{equation~\ref{#1}}

\def\1{\bm{1}}

\DeclareMathAlphabet{\mathsfit}{\encodingdefault}{\sfdefault}{m}{sl}
\SetMathAlphabet{\mathsfit}{bold}{\encodingdefault}{\sfdefault}{bx}{n}

%% file: appendix_worked_examples.tex
\providecommand{\lotsline}[1]{\noindent\parbox[t]{\linewidth}{\ttfamily\scriptsize\raggedright\hangindent=1em #1\strut}\par}
\providecommand{\lotsdel}[1]{\noindent\colorbox{red!13}{\parbox[t]{\dimexpr\linewidth-2\fboxsep}{\ttfamily\scriptsize\raggedright\hangindent=1em #1\strut}}\par}
\providecommand{\lotschg}[1]{\noindent\colorbox{orange!22}{\parbox[t]{\dimexpr\linewidth-2\fboxsep}{\ttfamily\scriptsize\raggedright\hangindent=1em #1\strut}}\par}
\providecommand{\lotskey}[1]{\textcolor{black!60}{\scriptsize\sffamily\bfseries #1}}
\newtcolorbox{lotsbox}[1]{enhanced,breakable,colback=white,colframe=black!80,colbacktitle=black!80,coltitle=white,
  fonttitle=\sffamily\bfseries\footnotesize,title={#1},boxrule=0.5pt,arc=1pt,left=5pt,right=5pt,top=4pt,bottom=4pt,
  before skip=7pt,after skip=7pt,fontupper=\small}

\subsection{Worked Examples of Tool-Space Optimization}
\label{app:worked}

This section follows one task through the procedure of \S\ref{sec:method} with the recorded data of one run.
The examples explain the mechanism; the evidence for accuracy is in \S\ref{sec:exp-shared} and \S\ref{sec:exp-loop}.
{\setlength{\fboxsep}{1pt}\colorbox{red!13}{Light red}} marks evidence removed by an intervention and {\setlength{\fboxsep}{1pt}\colorbox{orange!22}{light orange}} marks content that changes because of it.

\paragraph{TGB: one task from request scores to a reused space.}
The task is \texttt{fx\_settle} inside the task-arrival stream of Figure~\ref{fig:stream} (Qwen2.5-7B, seed~0, 120-tool registry), and the request is the first request of the task's first batch; no other selection was made.
Requests, answers, scores, stored scores, spaces, and accuracies are copied from the run log; tool outputs are regenerated by the benchmark's deterministic tool code, which the log does not store.

\begin{lotsbox}{1. Task and request}
\lotskey{Task type} \texttt{fx\_settle}: a price and a quantity retrieved by DocRetrieve, a rate from ExchangeRate, combined by Calculator. \\
\lotskey{Model / run} Qwen2.5-7B-Instruct; task-arrival stream, seed 0; batches of 4 requests, exploration $\varepsilon{=}0.25$, EMA weight $\alpha{=}0.3$. \\
\lotskey{Request} \texttt{fx\_settle\_0125}, update round 0 (the task's first batch). \\
\lotskey{Space before the update} the whole registry, 120 tools (15 authored, 105 synthetic distractors). \\
\lotskey{Budget} $K{=}3$, chosen beforehand on the task's fitting requests under the 120-tool registry by accuracy (top-$K$ accuracy $K{=}1$: 0.01, $K{=}2$: 0.00, $K{=}3$: 0.72, $K{=}4$: 0.51, $K{=}6$: 0.36).\par\smallskip
\lotskey{User request}\par
\lotsline{We are settling 12 units of Vexforn-X4 in the settlement currency, and 960.00 has already been prepaid. By how much does the settlement value exceed the amount already prepaid?}
\end{lotsbox}
\begin{lotsbox}{2. Tool outputs and realized answer}
On TGB the benchmark executes every served tool once per request and lists the outputs in the prompt; these are not calls chosen by the model, and the prompt contains no tool documentation.\par\smallskip
\lotskey{Outputs of the task's tool chain}\par
\lotsline{Calculator: 971.86}
\lotsline{DocRetrieve: DocRetrieve:}
\lotsline{[1] Vexforn-X4 --- Vexforn-X4 is supplied in single units at a listed price of 63.52 per unit.}
\lotsline{[2] Mirdan-X4 --- general background, no figures given.}
\lotsline{ExchangeRate: ExchangeRate: 1 unit = 1.275 settlement units}
\smallskip\lotskey{Outputs of four of the 117 other served tools} (113 further lines omitted)\par
\lotsline{GoogleSearch: No results found.}
\lotsline{CurrencyConvert: CurrencyConvert: no monetary amount found.}
\lotsline{Summarize: Summarize: the image contains no readable figures.}
\lotsline{TorqueIndex: TorqueIndex: 79.50}
\smallskip\lotskey{Realized answer $y$}\par
\lotsline{11.86}
\smallskip{\footnotesize Evaluation note, not part of any scoring input: the reference value is 11.86, so this answer is counted correct.}
\end{lotsbox}
\begin{lotsbox}{3. Leave-one-out intervention on \texttt{DocRetrieve} (excerpt)}
Both columns end in the same target string, which is teacher-forced; only the tool outputs differ.
On TGB removing a tool re-executes the chain without it, so the Calculator line changes as well; the other 116 output lines are identical in the two contexts.\par\smallskip
\noindent\begin{minipage}[t]{0.485\linewidth}\lotskey{Context $h_x$}\par
\lotsline{Answer the question with a short final answer only.}
\lotsline{\mbox{}}
\lotsline{Question: We are settling 12 units of Vexforn-X4 in the settlement currency, and 960.00 has already been prepaid. By how much does the settlement value exceed the amount already prepaid?}
\lotsline{\mbox{}}
\lotsline{Tool output:}
\lotsline{GoogleSearch: No results found.}
\lotschg{Calculator: 971.86}
\lotsdel{DocRetrieve: DocRetrieve:}
\lotsdel{[1] Vexforn-X4 --- Vexforn-X4 is supplied in single units at a listed price of 63.52 per unit.}
\lotsdel{[2] Mirdan-X4 --- general background, no figures given.}
\lotsline{ExchangeRate: ExchangeRate: 1 unit = 1.275 settlement units}
\lotsline{\mbox{}}
\lotsline{Final answer:}
\lotsline{\textbf{ 11.86}\quad$\leftarrow$ scored target $y$}
\end{minipage}\hfill
\begin{minipage}[t]{0.485\linewidth}\lotskey{Context $h_x^{-s}$, $s=$ DocRetrieve}\par
\lotsline{Answer the question with a short final answer only.}
\lotsline{\mbox{}}
\lotsline{Question: We are settling 12 units of Vexforn-X4 in the settlement currency, and 960.00 has already been prepaid. By how much does the settlement value exceed the amount already prepaid?}
\lotsline{\mbox{}}
\lotsline{Tool output:}
\lotsline{GoogleSearch: No results found.}
\lotschg{Calculator: NameError: name 'P' is not defined}
\lotsline{ExchangeRate: ExchangeRate: 1 unit = 1.275 settlement units}
\lotsline{\mbox{}}
\lotsline{Final answer:}
\lotsline{\textbf{ 11.86}\quad$\leftarrow$ same target $y$}
\end{minipage}
\end{lotsbox}
\begin{lotsbox}{4. Likelihood contributions for this request}
$\mathcal{L}_\theta$ is the mean log-likelihood of the tokens of the realized answer (here \texttt{` '}, \texttt{`1'}, \texttt{`1'}, \texttt{`.'}, \texttt{`8'}, \texttt{`6'}), located inside the full prompt; $I_x(s)=\mathcal{L}_\theta(y\mid h_x)-\mathcal{L}_\theta(y\mid h_x^{-s})$ is computed from unrounded values.\par\smallskip
\centering\footnotesize\begin{tabular}{@{}lrrrr@{}}\toprule
Tool $s$ & $\mathcal{L}_\theta(y\mid h_x)$ & $\mathcal{L}_\theta(y\mid h_x^{-s})$ & $I_x(s)$ & $I_x(s)$ in the run log \\ \midrule
DocRetrieve & -0.1316 & -2.6867 & +2.555 & +2.640 \\
ExchangeRate & -0.1316 & -2.7974 & +2.666 & +2.654 \\
Calculator & -0.1316 & -2.6290 & +2.497 & +2.554 \\
\midrule
Summarize & -0.1316 & -1.2136 & +1.082 & +1.050 \\
CurrencyConvert & -0.1316 & -0.8356 & +0.704 & +0.723 \\
GoogleSearch & -0.1316 & -0.3350 & +0.203 & +0.200 \\
TorqueIndex & -0.1316 & -1.2569 & +1.125 & +1.028 \\
\bottomrule\end{tabular}\par\smallskip\raggedright\small
The run log stores only $I_x(s)$. The two likelihood terms were recomputed for this appendix with the same prompt, answer-token positions, and normalization (HF transformers, CPU, bfloat16); the original run used vLLM on GPU, and the last column gives its value.
The three chain tools carry the score. Distractor scores are small and can be positive: the score measures how much the evidence supports the answer the model gave, whether or not that answer is correct.
\end{lotsbox}
\begin{lotsbox}{5. Aggregate over the batch and update the space}
The task's first batch has 4 requests, all served under the registry, so every tool has 4 scored requests; the table shows the whole batch for seven of the 120 tools. Two of the four realized answers are wrong, and their chain tools still score high.\par\smallskip
\centering\footnotesize\setlength{\tabcolsep}{3pt}\begin{tabular}{@{}lrrrrrrr@{}}\toprule
Request & \texttt{0125} & \texttt{0385} & \texttt{0034} & \texttt{0175} & Scored & Batch mean & Rank \\
Realized answer & {\scriptsize\texttt{11.86}} & {\scriptsize\texttt{49.26}} & {\scriptsize\texttt{22.52}} & {\scriptsize\texttt{280.21 - 280.00 = 0.21}} & requests & $\bar I_t(s)$ & of 120 \\
Counted correct & yes & no & no & yes & & & \\ \midrule
DocRetrieve & +2.640 & +3.278 & +2.175 & +0.746 & 4 & +2.210 & 3 \\
ExchangeRate & +2.654 & +2.775 & +3.136 & +0.788 & 4 & +2.338 & 1 \\
Calculator & +2.554 & +2.280 & +3.576 & +0.807 & 4 & +2.304 & 2 \\
\midrule
Summarize & +1.050 & -0.097 & +0.041 & -0.010 & 4 & +0.246 & 7 \\
TorqueIndex & +1.028 & -0.118 & +0.036 & -0.002 & 4 & +0.236 & 11 \\
CurrencyConvert & +0.723 & -0.061 & +0.167 & +0.054 & 4 & +0.221 & 23 \\
GoogleSearch & +0.200 & -0.099 & +0.046 & -0.014 & 4 & +0.033 & 120 \\
\bottomrule\end{tabular}\par\smallskip\raggedright\small
\lotskey{Update} top-$K{=}3$ of the ranking: \texttt{Calculator, DocRetrieve, ExchangeRate}. \lotskey{Space} 120 tools $\rightarrow$ 3 tools.\par\smallskip
\lotskey{Next batch} Served under the three-tool space (no request was explored). The stored score is the batch mean after the first batch and an exponential moving average afterwards, $0.3\times$ new batch mean $+\,0.7\times$ stored; a tool that the batch did not expose keeps its stored score.\par\smallskip
\centering\footnotesize\begin{tabular}{@{}lrrr@{}}\toprule Tool & Stored after batch 1 & Batch-2 mean & Stored after batch 2 \\ \midrule
ExchangeRate & 2.3385 & 3.2876 & 2.6232 \\
Calculator & 2.3041 & 2.4274 & 2.3411 \\
DocRetrieve & 2.2097 & 2.2685 & 2.2273 \\
MarginCalc (rank 4, not exposed) & 0.2655 & --- & 0.2655 \\ \bottomrule\end{tabular}\par
\end{lotsbox}
\begin{lotsbox}{6. Reuse on a subsequent request}
\lotskey{Request} \texttt{fx\_settle\_0351}, update round 1: served after the space was built, under that space (Calculator, DocRetrieve, ExchangeRate).\par
\lotsline{We are settling 14 units of Kalond-R3 in the settlement currency, and 540.00 has already been prepaid. By how much does the settlement value exceed the amount already prepaid?}
\smallskip\lotskey{Tool outputs in the prompt (complete)}\par
\lotsline{Calculator: 556.57}
\lotsline{DocRetrieve: DocRetrieve:}
\lotsline{[1] Kalond-R3 --- Kalond-R3 is supplied in single units at a listed price of 36.24 per unit.}
\lotsline{[2] Kalond-Q3 --- general background, no figures given.}
\lotsline{ExchangeRate: ExchangeRate: 1 unit = 1.097 settlement units}
\smallskip\lotskey{Realized answer}\par
\lotsline{16.57}
\smallskip{\footnotesize Evaluation note: the reference value is 16.57, so the answer is counted correct. This request was not also served under Keep all. On the task's 40 fixed evaluation requests the three-tool space scores 0.675 and Keep all 0.225; the second batch is answered 3 of 4 correctly with 143 prompt tokens per request against 1830 under the registry.}
\end{lotsbox}

\pagebreak[3]\paragraph{BFCL.}
The second example uses the per-class protocol of Table~\ref{tab:bfcl-models}: Qwen3-4B-Instruct, API class \texttt{TradingBot}, fitted on the class's 10 fitting records; the record is the class's first fitting record and the turn its first turn with a tool call before the final message.

\begin{lotsbox}{B1. Recorded turn, scoring target, and intervention (excerpt)}
\lotskey{Record} \texttt{multi\_turn\_base\_100}, turn 0. The scored target is the last assistant message of the turn; it is never removed. Removing \texttt{get\_stock\_info} deletes the assistant message that calls it and the observation that call produced; the system prompt with the tool menu is identical in both contexts and is omitted here.\par\smallskip
\noindent\begin{minipage}[t]{0.485\linewidth}\lotskey{Context $h_x$ (end of the turn)}\par
\lotsline{user: I'm contemplating enhancing my investment portfolio with some tech industry assets, and I've got my eye on Nvidia. I'm keen to know its current stock price, and would appreciate if you could source this information for me.}
\lotsdel{assistant: [get\_stock\_info(symbol="NVDA")]}
\lotsdel{tool: \{"price": 220.34, "percent\_change": 0.34, "volume": 1.234, "MA(5)": 220.45, "MA(20)": 220.67\}}
\lotsline{\textbf{assistant: The current stock price of Nvidia (NVDA) is \$220.34, with a 0.34\% increase in price today. The trading volume is 1.234 million shares, and its 5-day Moving Average (MA(5)) is \$220.45, while the 20-day Moving Average (MA(20)) is \$220.67.}}
\end{minipage}\hfill
\begin{minipage}[t]{0.485\linewidth}\lotskey{Context $h_x^{-s}$, $s=$ get\_stock\_info}\par
\lotsline{user: I'm contemplating enhancing my investment portfolio with some tech industry assets, and I've got my eye on Nvidia. I'm keen to know its current stock price, and would appreciate if you could source this information for me.}
\lotsline{\textbf{assistant: The current stock price of Nvidia (NVDA) is \$220.34, with a 0.34\% increase in price today. The trading volume is 1.234 million shares, and its 5-day Moving Average (MA(5)) is \$220.45, while the 20-day Moving Average (MA(20)) is \$220.67.}}
\end{minipage}\par\smallskip
{\footnotesize Logged score of \texttt{get\_stock\_info} on this record (mean over the record's scored turns): +0.147. The log keeps this record-level difference only; the two likelihood terms and the per-turn values were not stored.}
\end{lotsbox}
\begin{lotsbox}{B2. Full and demoted documentation of one tool}
\texttt{get\_transaction\_history} is ranked 27 of 85 for this class, below $K{=}20$, so it is served with demoted documentation (function \texttt{\_compress} of the implementation). The name, the parameter names, types and defaults, and the required list are kept; the class introduction and the per-parameter prose are removed; the first sentence of the tool's own description is kept. The \texttt{response} field is unchanged and omitted here.\par\smallskip
\noindent\begin{minipage}[t]{0.485\linewidth}\lotskey{Full documentation (265 tokens)}\par
\lotsline{"name": "get\_transaction\_history"}
\lotsdel{"description": "This tool belongs to the trading system, which allows users to trade stocks, manage their account, and view stock information.}
\lotsline{  Tool description: Get the transaction history within a specified date range."}
\lotsline{"start\_date": \{"type": "string", "default": "None"\}}
\lotsdel{    "description": "Start date for the history (format: 'YYYY-MM-DD')."}
\lotsline{"end\_date": \{"type": "string", "default": "None"\}}
\lotsdel{    "description": "End date for the history (format: 'YYYY-MM-DD'). "}
\lotsline{"required": []}
\end{minipage}\hfill
\begin{minipage}[t]{0.485\linewidth}\lotskey{Demoted documentation (202 tokens)}\par
\lotsline{"name": "get\_transaction\_history"}
\lotsline{"description": "Get the transaction history within a specified date range."}
\lotsline{"start\_date": \{"type": "string", "default": "None"\}}
\lotsline{"end\_date": \{"type": "string", "default": "None"\}}
\lotsline{"required": []}
\end{minipage}\par\smallskip
{\footnotesize Token counts are of this tool's JSON definition under the Qwen3-4B tokenizer; they are not the cost of a request.}
\end{lotsbox}
\begin{lotsbox}{B3. Resulting task configuration}
\raggedright\lotskey{Full documentation} the top $K{=}20$ of the class's 85 ranked tools: 12 tools of the trading API (\texttt{get\_order\_details, get\_watchlist, get\_stock\_info, get\_symbol\_by\_name, place\_order, add\_to\_watchlist, get\_account\_info, fund\_account, cancel\_order, get\_available\_stocks, get\_current\_time, filter\_stocks\_by\_price}) and 8 tools of other APIs that appear in the class's records.\par\smallskip
\lotskey{Demoted documentation} every other tool in a record's menu, including 8 tools of the trading API (\texttt{get\_transaction\_history, get\_order\_history, trading\_login, notify\_price\_change, trading\_get\_login\_status, remove\_stock\_from\_watchlist, withdraw\_funds, trading\_logout}).\par\smallskip
\textbf{Effective tools remain callable.} \par\smallskip
{\footnotesize On this class's 40 evaluation records (three seeds): LOTS 47.5/47.5/47.5 (mean 47.5); random ranking at the same budget 50.0/32.5/42.5 (mean 41.7); Keep all, mean of three seeds, 40.8. $K$ was chosen on the 10 fitting records.}
\end{lotsbox}